\documentclass[11pt]{article}

\usepackage[final]{acl}
\usepackage{times}
\usepackage{latexsym}
\usepackage{inconsolata}
\usepackage[utf8]{inputenc} 
\usepackage[T1]{fontenc}    
\usepackage{hyperref}       
\usepackage{url}            
\usepackage{booktabs}       
\usepackage{cuted}    
\usepackage{caption}  
\usepackage{amsfonts}       
\usepackage{nicefrac}       
\usepackage{microtype}      
\usepackage{subcaption}    
\usepackage{graphicx}       
\usepackage{amsmath}        
\usepackage{algorithm}      
\usepackage{algorithmic}    
\usepackage{natbib}         
\usepackage{float}
\usepackage{multirow}
\usepackage{makecell}
\usepackage{caption}
\usepackage{placeins}
\usepackage{booktabs}
\usepackage{multirow}
\usepackage{makecell}
\usepackage{adjustbox}
\usepackage[table]{xcolor}
\usepackage{booktabs}
\usepackage{multirow}
\usepackage{makecell}
\usepackage{graphicx}
\usepackage{booktabs}
\usepackage{caption}
\usepackage{array}
\usepackage[table]{xcolor}
\usepackage{caption}
\usepackage{booktabs}
\usepackage{multirow}
\usepackage{makecell}
\usepackage{graphicx}
\usepackage{algorithm}
\usepackage{algorithmic}
\usepackage{caption}
\usepackage{hyperref}
\usepackage{amsthm}
\usepackage{booktabs}
\usepackage{graphicx}
\usepackage{amsmath,amssymb}
\usepackage{booktabs}
\usepackage{tabularx}
\usepackage{xcolor}
\usepackage{makecell}
\usepackage{graphicx}
\usepackage{booktabs}
\usepackage{multirow}
\usepackage{tabularx}
\usepackage{float}
\usepackage{cuted}      
\usepackage{caption}    
\usepackage{placeins}   
\usepackage{placeins}
\usepackage{array}
\usepackage{placeins}
\usepackage{caption}
\usepackage{booktabs}
\usepackage{graphicx}
\usepackage{algorithm}
\usepackage{amsmath}
\usepackage{amssymb}
\usepackage{amsthm}

\newtheorem{lemma}{Lemma}
\newtheorem{proposition}{Proposition}

\definecolor{fdblue}{RGB}{232,241,252}
\definecolor{bbgreen}{RGB}{226,245,226}
\definecolor{lightgrayrow}{RGB}{247,247,247}
\definecolor{nfeorange}{RGB}{210,105,0}
\definecolor{nfeblue}{RGB}{0,55,170}
\definecolor{bbtextgreen}{RGB}{20,110,45}

\newcommand{\fdcell}[1]{\cellcolor{fdblue}{#1}}
\newcommand{\bbcell}[1]{\cellcolor{bbgreen}{#1}}
\newcommand{\nfe}[1]{\textcolor{nfeblue}{#1}}

\title{BlockBatch: Multi-Scale Consensus Decoding for Efficient Diffusion
Language Model Inference}

\author{
Xiaoyou Wu \quad
Cheng-Jhih Shih \quad
Binfei Ji \quad
Yong Liu \quad
Yingyan (Celine) Lin \\
Georgia Institute of Technology \\
\texttt{xwu488@gatech.edu}
}

\makeatletter
\newcommand\blfootnote[1]{%
  \begingroup
  \renewcommand\thefootnote{}%
  \footnotetext{#1}%
  \addtocounter{footnote}{-1}%
  \endgroup
}
\makeatother

\begin{document}
\maketitle
\blfootnote{GitHub code available at \href{https://github.com/Laurence-Wu/BlockBatch}{\texttt{BlockBatch}}.}

\raggedbottom

\begin{abstract}

Diffusion language models (dLLMs) generate text by iteratively denoising multiple token positions in parallel, offering an attractive alternative to strictly autoregressive decoding. In practice, however, block-wise dLLM inference exposes a difficult granularity trade-off: small blocks preserve local conditioning but require many denoising steps, whereas large blocks expose more parallelism but can make premature commitments and accumulate cache error. Existing acceleration methods typically choose a single block size per request, leaving the complementarity among block sizes unused.
We show that block size itself is a useful branching dimension. Different block sizes induce related but non-identical KV-cache trajectories: branches often share an initial prefix, bifurcate at semantically decisive positions, and later agree on syntactically lightweight tokens. Motivated by this structure, we propose BlockBatch, a training-free online inference framework that executes multiple block-size branches for the same request inside a batched forward pass. 
BlockBatch coordinates these branches through confidence-gated token merging, leader-based synchronization, and periodic full-sequence refreshes that re-anchor local block updates to a globally consistent KV state. 
Across 3 representative dLLMs and 4 datasets, BlockBatch reduces denoising NFEs by 26.6\% on average and achieves a 1.33$\times$ average end-to-end speedup over Fast-dLLM while preserving accuracy.
These results identify block-size diversity as a practical and previously underexplored axis for branch-parallel dLLM inference.

\end{abstract}

\section{Introduction}

\begin{figure}[t]
    \centering
    \includegraphics[width=1\linewidth]{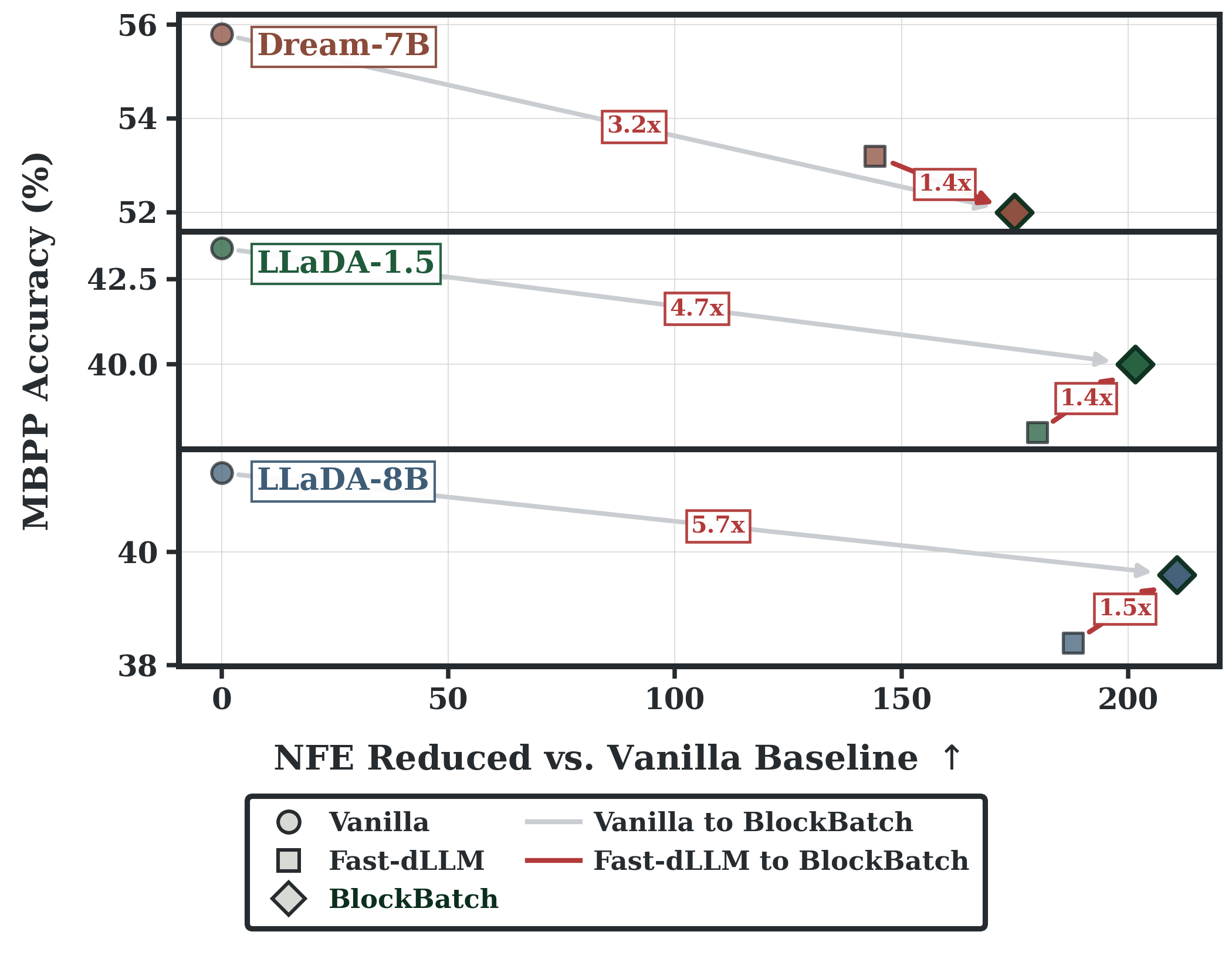}
    \vspace{-2em}
    \caption{
    \textbf{BlockBatch speedup against Fast-dLLM and the vanilla baseline on MBPP.}
}
    \label{fig:teaser}
\end{figure}

Diffusion language models (dLLMs)~\citep{gong2025scalingdiffusionlanguagemodels} replace left-to-right factorization with iterative
masked denoising~\citep{campbell2022continuoustimeframeworkdiscrete,campbell2024generativeflowsdiscretestatespaces,chen2023analogbitsgeneratingdiscrete,vignac2023digressdiscretedenoisingdiffusion}: at each generation round, the model predicts distributions over
many token positions and commits a subset of tokens in parallel. Recent masked
dLLMs~\citep{ye2025dream,nie2025large} have substantially narrowed the gap to autoregressive
language models and, in several regimes, exhibit competitive scaling,
instruction-following, and flexible generation behavior~\citep{sahoo2024simple,nie2025large,ye2025dream,prabhudesai2025diffusion}.
This parallelism is attractive for inference, but it is not free. Committing too few
tokens preserves local conditioning but requires many denoising rounds; committing
too many tokens increases apparent parallelism but can make premature decisions
and degrade cache consistency~\citep{ma2025dkvcache}. Efficient dLLM inference hence depends on
finding the right granularity at which to trade off parallel token updates,
confidence, and state reuse.
\begin{figure*}[t]
    \centering
    \includegraphics[width=0.92\textwidth]{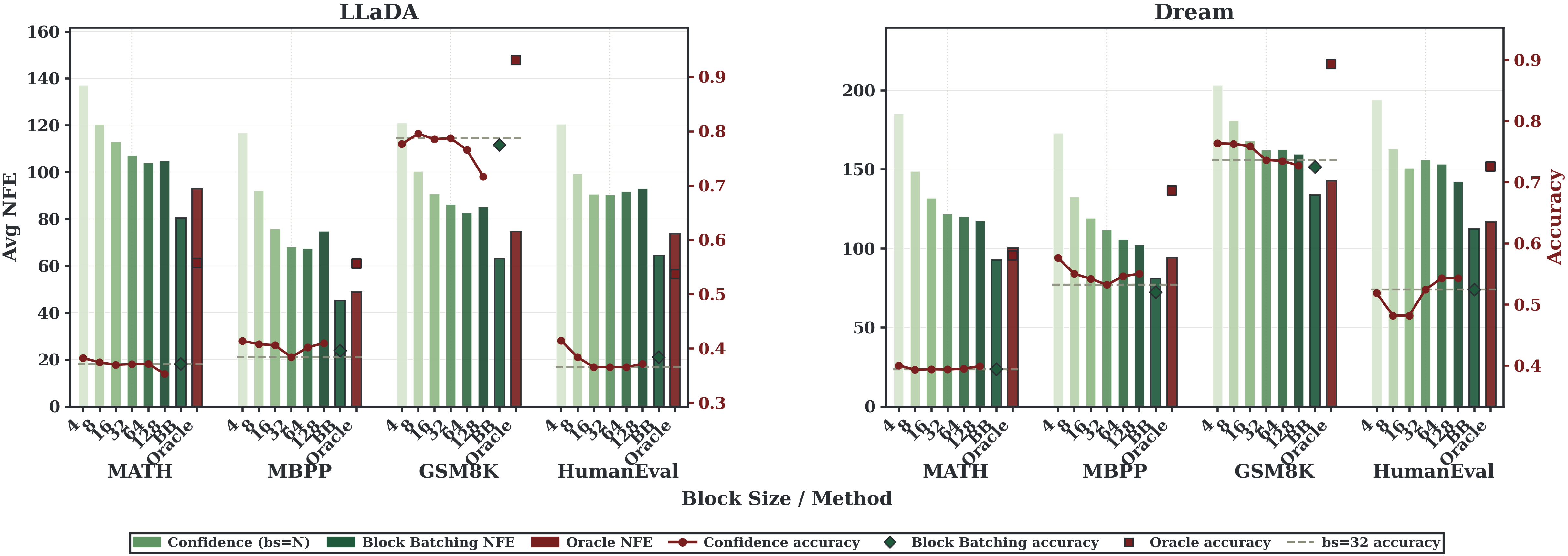}
    \vspace{-0.5em}
    \caption{
\small
\textbf{Per-sample block-size oracle versus fixed block-size decoding.}
Bars report average NFE and curves/markers report accuracy for fixed confidence-decoding block sizes, block batching, and an oracle that selects the best block size per prompt.
The oracle improves the accuracy--NFE tradeoff across models and tasks, motivating block size as a per-sample branching dimension rather than a fixed hyperparameter.
}
    \label{fig:oracle_ablation}
\end{figure*}
A central mechanism for making this tradeoff practical is block-wise inference
with approximate KV reuse~\citep{jiang2026d2cache}. Instead of recomputing the full sequence at every
denoising step, block-wise methods update a local decoding window while reusing,
delaying, or selectively refreshing cached KV states outside the active region~\citep{wu2026fastdllm,liu2025dllmcache,jiang2026d2cache}.
The block size determines the granularity of this process. Small blocks are
conservative: they more closely approximate sequential conditioning but 
require more denoising function evaluations (NFEs)~\citep{cheng2025sdarsynergisticdiffusionautoregressionparadigm}. Large blocks expose more
parallel work per model invocation, but they weaken intra-block conditioning,
accumulate larger cache discrepancies, and may require downstream correction~\citep{nie2025scaling,zhang2025diffusionvsautoregressivelanguage}.
Existing systems therefore typically select a single block size for a request,
either as a fixed hyperparameter or through a local adaptation rule. This
single-trajectory view leaves a natural opportunity unused: different block sizes
may be useful for different parts of the same generation.

Our starting point is that block size is not merely an inference hyperparameter,
but a source of complementary decoding trajectories. In our experiments shown in Fig.~\ref{fig:oracle_ablation}, a
per-sample oracle that chooses among block sizes achieves a better
accuracy--NFE tradeoff than any single fixed block-size configuration. This
indicates that there is no globally optimal block size across prompts, tasks, or
models. The observation suggests a different form of branch-parallel decoding.
Prior acceleration methods for autoregressive language models often branch along
the token-candidate axis, producing speculative or tree-structured continuations
that are later verified, accepted, or discarded
\citep{leviathan2023fast,cai2024medusa}. 
In contrast, we branch along the block-size
axis. All branches share the same prompt and initial KV cache, but they advance
with different denoising granularities and therefore trace different trajectories
through the space of partially decoded sequences and KV-cache states.

These block-size branches are neither independent generations nor redundant
copies. At the token level, they often agree on an initial prefix, bifurcate at a
small number of semantically decisive positions, and later re-agree at positions
with low local ambiguity. Our token-category analysis shows that such later-stage
agreement is concentrated on syntactically lightweight tokens (see Appendix \ref{app:later-stage-consensus}). In contrast, bifurcation positions tend to carry the semantic content
that determines the final answer. 
At the KV-cache level, block denoising acts as
a local update: it modifies only the active region while relying on cached context
elsewhere, where a full-sequence refresh re-anchors the branch to a globally consistent KV state. Together, these observations suggest
that useful multi-branch dLLM inference should preserve branch diversity around
bifurcation points, exchange information only when branches are compatible, and
periodically repair stale cache states.

Motivated by the above findings, we propose BlockBatch, a training-free
online inference framework for efficient dLLM decoding. For each request, BlockBatch instantiates multiple branches with different block sizes and executes
their active windows inside a single batched model invocation. The branches are
coordinated by three important operations. First, a confidence-gated merge
transfers token proposals only between compatible branches and only when the
destination branch assigns sufficient probability to the proposed token. Second,
a leader-based synchronization rule prevents lagging branches from wasting
computation by copying the sequence state and KV cache of a sufficiently advanced
branch. Third, periodic full-sequence refreshes recompute KV states for active
branches, correcting accumulated drift from local block updates. Large-block
branches provide aggressive progress, small-block branches provide conservative
alternatives, and merge/sync/full-sequence refresh operations prevent the search from
degenerating into unrelated or stale trajectories.

In summary, 
our contributions are as follows:
\begin{itemize}
    \item \textbf{Block-size diversity as branch parallelism.}
    We identify block size as an underexplored axis for multi-branch dLLM
    inference and show that different block sizes provide complementary
    accuracy--efficiency tradeoffs across samples.

    \item \textbf{Token- and KV-level characterization.}
    We analyze how block-size branches relate to one another, revealing
    token-level bifurcation and later-stage consensus as well as KV-cache
    dynamics in which local block updates are periodically corrected by
    full-sequence refreshes.

    \item \textbf{A training-free inference framework.}
    We introduce BlockBatch, which executes multiple block-size branches in
    parallel and coordinates them through confidence-gated merging,
    leader-based synchronization, and full-sequence KV refresh.

    \item \textbf{Empirical validation.}
    Experiments on three dLLMs (LLaDA-1.5-8B, LLaDA-Instruct-8B, Dream-7B) across four benchmarks (GSM8K, MATH, HumanEval, MBPP) show that BlockBatch reduces denoising steps by 26.6\% and achieves a 1.33$\times$ end-to-end speedup on average over Fast-dLLM while preserving accuracy.
\end{itemize} 

\section{Preliminaries}

\subsection{Diffusion LLMs as a Discrete Markov Process}

Unlike continuous diffusion models that operate on Gaussian noise
\citep{sohl2015deep, song2019score, ho2020ddpm, song2020ddim, song2021sde},
Masked Diffusion Language Models (MDLMs) operate in a discrete, categorical
state space
\citep{hoogeboom2021argmax, austin2021d3pm, nie2025large,
hoogeboom2021argmaxflowsmultinomialdiffusion}.
Let $\mathcal{V}$ denote the vocabulary and let $\texttt{[M]}$ denote the
special \texttt{[MASK]} token. The extended vocabulary is
$\tilde{\mathcal{V}} = \mathcal{V} \cup \{\texttt{[M]}\}$. For a sequence of
length $L$, the state at generative step $t$ is represented as
$\mathbf{X}_t \in \tilde{\mathcal{V}}^L$.

The generative reverse process transitions from a fully masked state
$\mathbf{X}_T = [\texttt{[M]}, \dots, \texttt{[M]}]$ to a decoded state
$\mathbf{X}_0 \in \mathcal{V}^L$
\citep{sahoo2024simple,
ho2020ddpm}.
At each step $t$, the network observes $\mathbf{X}_t$ and outputs a categorical
distribution $p_\theta(x^{(i)} = v \mid \mathbf{X}_t)$ for each position $i$.
Standard diffusion LLMs decode this distribution using a confidence-based
mechanism. The top-1 confidence score is
$c_{t,i} = \max_{v \in \mathcal{V}} p_\theta(v \mid \mathbf{X}_t)$.
Given a confidence threshold $\tau_{\mathrm{conf}}$ and the most confident masked index
$i^*$, the deterministic transition for position $i$ is
\begingroup
\setlength{\abovedisplayskip}{3pt}
\setlength{\belowdisplayskip}{3pt}
\begin{equation}
\makebox[\columnwidth][c]{%
\resizebox{1.03\columnwidth}{!}{$\displaystyle
x_{t-1,i}
=
\begin{cases}
\arg\max\limits_{v \in \mathcal{V}} p_\theta(v \mid \mathbf{X}_t),
&
\text{if } x_{t,i}=\texttt{[M]} \text{ and } 
\left(c_{t,i}\geq\tau_{\mathrm{conf}} \lor i=i^*\right), \\[0.25em]
x_{t,i},
& \text{otherwise.}
\end{cases}
$}%
}
\end{equation}
\endgroup

We can formalize the overall transition probability as
$P(\mathbf{X}_{t-1} \mid \mathbf{X}_t)$. The sequence therefore forms a
discrete-time Markov chain:
\begingroup
\setlength{\abovedisplayskip}{3pt}
\setlength{\belowdisplayskip}{3pt}
\begin{equation}
\resizebox{0.98\columnwidth}{!}{$\displaystyle
P(\mathbf{X}_{t-1} \mid \mathbf{X}_t, \mathbf{X}_{t+1}, \dots, \mathbf{X}_T)
=
P(\mathbf{X}_{t-1} \mid \mathbf{X}_t).
$}
\end{equation}
\endgroup

\subsection{KV Cache in Diffusion LLMs}

For batched inputs, prefix caching provides the first source of reuse by sharing
the KV states of identical prompt prefixes across requests. This reuse pattern is
standard in autoregressive Transformer serving, where KV caching avoids
recomputing previous key--value activations during decoding
\citep{vaswani2017attention, shazeer2019fast, dao2022flashattention,
kwon2023pagedattention}. Modern LLM serving systems further improve prefix reuse
and cache sharing through paged cache allocation, radix-tree or prefix-tree cache
layouts, and shared-prefix attention kernels.

Under the semi-autoregressive decoding paradigm,
Fast-dLLM~\citep{wu2026fastdllm} introduces block-wise approximate KV caching:
previously decoded blocks are reused, while the current block is recomputed
during denoising. Several follow-up methods refine which cache entries should be
recomputed and which should be retained across steps. dLLM-Cache
\citep{liu2025dllmcache} observes that dLLM inference contains a mostly static
prompt and a partially dynamic response, and applies long-interval prompt
caching together with feature-similarity-guided response updates. dKV-Cache
exploits token-dependent representation dynamics through delayed KV reuse,
storing decoded-token KV states one step later and providing decode and greedy
variants that trade quality for speed~\citep{ma2025dkvcache}. d$^2$Cache
selects masked tokens using confidence and a certainty prior, then selects
prompt and decoded tokens using attention-aware importance before updating only
the selected KV states \citep{jiang2026d2cache}. Related dLLM acceleration
methods also exploit guided unmasking, local determinism, block-diffusion
schedules, or hierarchical caching to reduce denoising cost
\citep{hu2025freecache, kong2025localleap, fu2025bitsrounds,
wu2025fastdllmv2}.

More broadly, these dLLM cache designs are connected to a larger body of
KV-cache management work in autoregressive LLM inference. Prior systems retain
important tokens, evict or compress stale cache entries, quantize KV states, or
fetch only query-relevant cache pages
\citep{zhang2023h2o, liu2023scissorhands, xiao2024streamingllm,
li2024snapkv, cai2024pyramidkv, tang2024quest, liu2024kivi,
lee2024infinigen}. These methods can be viewed through a common systems lens:
selecting which KV states to reuse and which to refresh, then balancing stale and
refreshed cache states to trade accuracy against latency.
\section{KV-Level and Token-Level Characterizations of dLLM Inference}
\label{sec:characterization}

\subsection{Block-Size Diversity as a Branching Axis}
\label{sec:blocksize-diversity}

Existing dLLM inference systems pick a single block size per request, treating it as a hyperparameter. However, this design choice is suboptimal, as the accuracy-maximizing block size varies substantially across prompts.
To quantify this, we evaluate a \emph{per-sample oracle} that selects, for each prompt, the block size in $\{4, 8, 16, 32, 64, 128\}$ that yields the highest accuracy under confidence-based decoding in Fig.~\ref{fig:oracle_ablation}. The oracle is not realizable at inference time, but it upper-bounds any  sample-wise block-size selection policy built on top of confidence decoding. 

Fig.~\ref{fig:oracle_ablation} reports NFE and accuracy on MATH~\citep{hendrycksmath2021} and MBPP~\citep{austin2021program} for LLaDA-Instruct-8B~\citep{nie2025large} and Dream-Base-7B~\citep{ye2025dream}. The oracle outperforms every fixed block size on both models and tasks while using fewer NFEs than the smaller block-size configurations, confirming that block-size diversity carries genuine per-sample information.
These observations motivate treating \textbf{block size} as a \textbf{branching dimension} rather than a hyperparameter. Notably, BlockBatch achieves even lower NFE while preserving accuracy through a per-step policy of sync and merge, which is a finer granularity version of the per sample policy, demonstrating that branching dimension is still an exploitable dimension.

\subsection{KV-Level Characterization: KV-Cache Trajectory Diagnostics}
\label{sec:kv-level-characterization}

\begin{figure}[t]
  \centering
  \includegraphics[width=\linewidth]{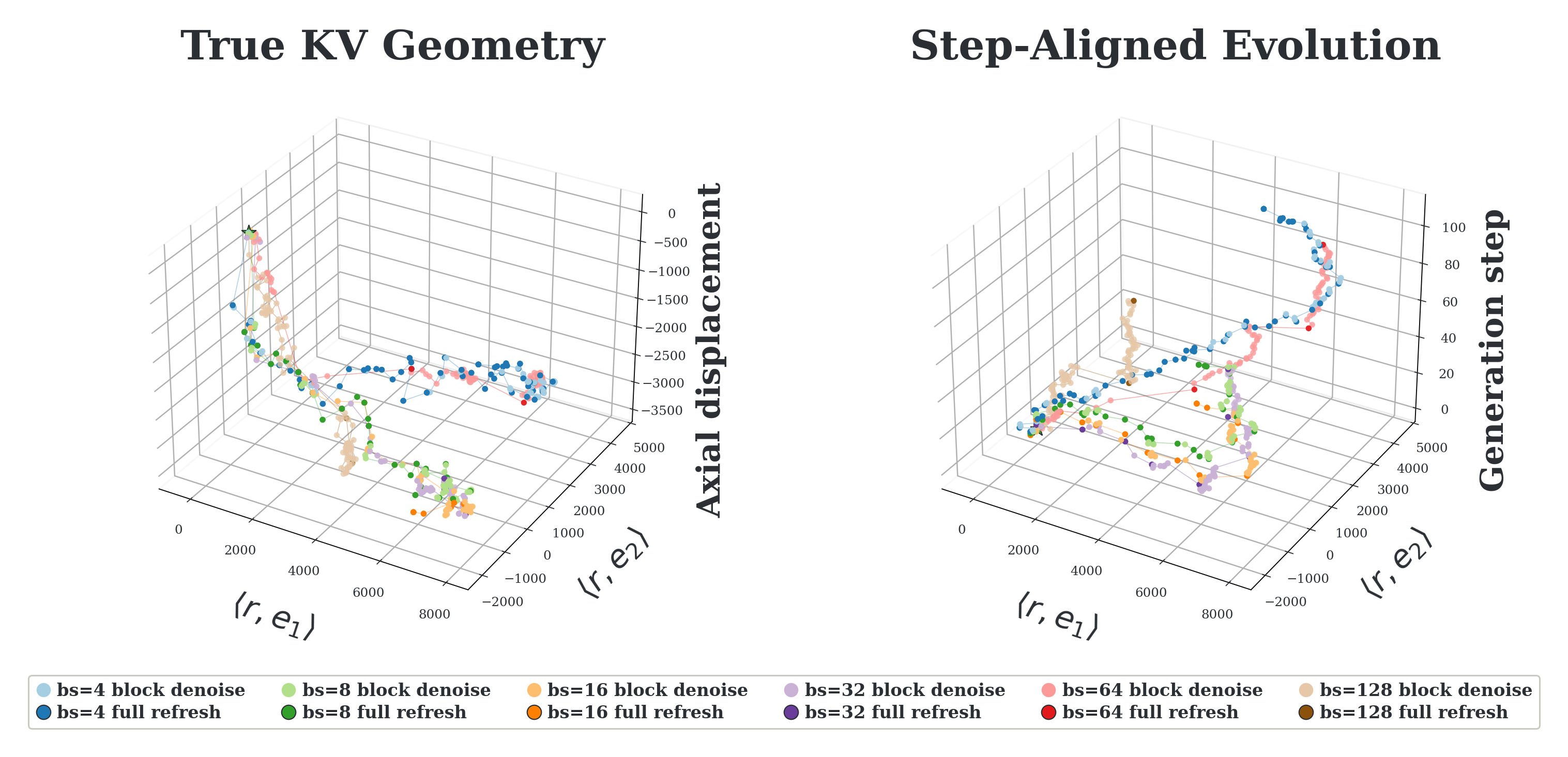}
  \caption{
  \textbf{KV-cache trajectory diagnostic on \texttt{HumanEval} sample~3.}
  Both panels visualize logged KV-cache states in the same tangent coordinate
  system around the prompt anchor $c_0$. The horizontal axes are the tangent
  projections $(\langle r,e_1\rangle,\langle r,e_2\rangle)$, where
  $e_1$ and $e_2$ are the leading SVD directions of the centered residuals.
  \textbf{Left: True KV Geometry.}
  The vertical axis shows axial displacement from the prompt anchor. As
  generation progresses, different block-size branches bifurcate and spread
  across the tangent plane, supporting the tangent-space bifurcation behavior
  described in Proposition~3 / Sec.~\ref{prop:Branch_bifurcation_tanspace}.
  \textbf{Right: Step-Aligned Evolution.}
  The vertical axis is replaced by generation step, revealing the temporal
  structure of the same trajectories. Adjacent block-denoise updates remain
  locally clustered, whereas large trajectory jumps are introduced by
  full-sequence refreshes. This pattern indicates that refresh operations,
  rather than local block denoising alone, are the main source of large KV-cache
  trajectory displacement, consistent with
  Proposition~1 / Sec.~\ref{prop:moreUpdateFromFullRefresh}.
  }
  \label{fig:kv-trajectory-humaneval3}
\end{figure}

While Fig.~\ref{fig:oracle_ablation} shows that block size is a useful
branching axis, the block-size parameter itself is not the final explanation
for the observed differences. Instead, we show that the underlying mechanism
through which block size changes NFE and accuracy is the trajectory induced in
KV-cache space.

Different block sizes change how many token positions are updated under a
shared cached context, how long a branch relies on local cache reuse, and how
frequently the branch is re-anchored by a full-sequence refresh. Thus, the
effect of block size is expressed through the sequence of KV-cache states
visited during inference.

Let $K_t^{(b)}$ denote the flattened KV-cache state of branch $b$ at event
$t$. At the KV level, block-wise dLLM inference alternates between a
\emph{block-denoise} update,
$K_{t+1}^{(b)}=K_t^{(b)}+\Delta_{\mathrm{blk}}(K_t^{(b)},x_t^{(b)},B_t^{(b)})$,
and a \emph{full-sequence refresh}, $K_{t+1}^{(b)}=F(x_t^{(b)})$. In the first
operation, only the active block is recomputed while the model attends to cached
KV states outside the current window. In the second operation, the KV cache is
recomputed from the entire current sequence. Appendix~\ref{app:kv-cache-vector-space}
formalizes these operations. In particular, Proposition~1
(Sec.~\ref{prop:moreUpdateFromFullRefresh}) shows that block-denoise updates should produce smaller
local movements in KV-cache vector space than full-sequence refreshes, because the direct
update is restricted to the active block.

This distinction is visible in Fig.~\ref{fig:kv-trajectory-humaneval3}. The
projected trajectories show that adjacent block-denoise steps move locally within a
compact region of the KV space, whereas full-sequence refreshes produce larger
global corrections between clusters. This supports the interpretation that
block-size branches are not independent generations but rather different local
traversals of a shared KV-cache geometry. Larger blocks tend to make more
aggressive local progress under reused cache states, while smaller blocks follow
more conservative trajectories with finer-grained conditioning.

The same view also explains why periodic refresh is necessary. A block-denoise step is efficient but
may increase this error because only a local window is updated. A full-sequence refresh
is more expensive but contracts the error by recomputing the cache from the full
token state. Proposition~2 (Sec.~\ref{refresh_stablize}) shows that if a
refresh is applied every $R$ block-denoise steps, the cache-consistency error
remains bounded when $\beta(1+\lambda_b)^R<1$. This gives a KV-space
explanation for the refresh interval used by BlockBatch: full-sequence refreshes are not
merely an implementation detail, but rather the mechanism that prevents local block
updates from drifting too far from the model-consistent KV manifold.\subsection{Token-Level Characterization: Bifurcation Tokens and Later-Stage Consensus}
\label{sec:consensus}

\begin{figure}[t]
\centering
\vspace{-0.4em}

\includegraphics[width=1.0\linewidth]{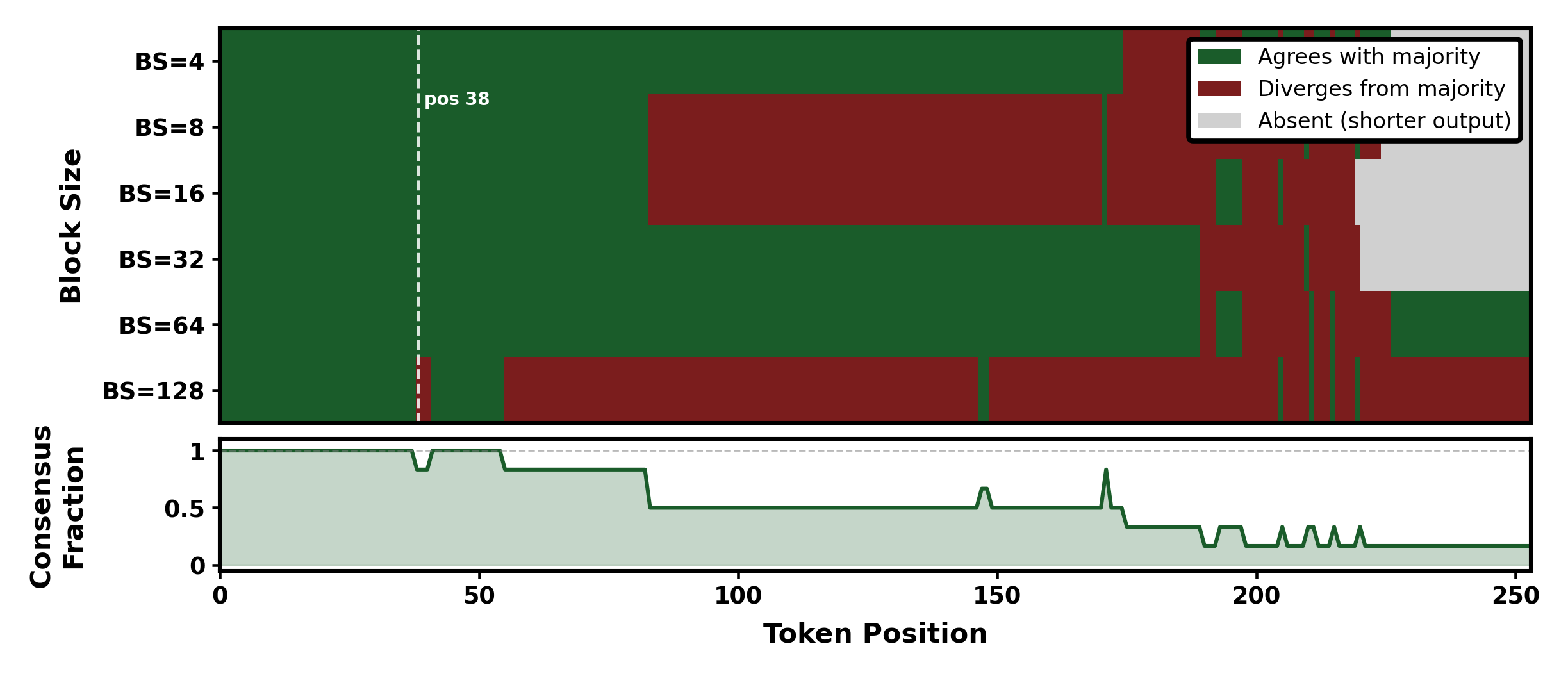}

\vspace{-1em}
\caption{
\small
Block-size branch consensus on HumanEval sample 3. Branches share an identical prefix for 37 tokens, after which BS = 128 bifurcates first and smaller blocks split off later. Isolated green columns inside red regions (e.g. pos 43, 150, 170) are later-stage consensus events — previously diverged branches briefly re-agree on a token while their KV trajectories remain distinct, as shown in Fig.~\ref{fig:kv-trajectory-humaneval3}.
}
\label{fig:consensus_humaneval_3}
\vspace{-0.6em}
\end{figure}
During block-batched inference, branches initialized with different block sizes begin from the same prompt and prefill cache, so their early decoded tokens often remain identical (though the length of the shared prefill varies across block sizes; see Tab.~\ref{tab:bifurcate-length} in Appendix~\ref{app:token-level}). Divergence typically occurs at specific \emph{bifurcation tokens}: once two branches commit different tokens at the same position, their subsequent token distributions and KV-cache trajectories begin to separate, since later predictions are conditioned on the current partially decoded sequence and its associated cache state. 

Interestingly, this token-level divergence does not always persist monotonically. Even after an early bifurcation, branches may later produce the same token at the same position. We refer to this phenomenon as \emph{later-stage consensus}. As shown in Fig.~\ref{fig:consensus_humaneval_3}, later-stage consensus appears around token position 190, highlighted in a deeper green.

However, later-stage consensus should not be interpreted as full trajectory convergence. Agreement on a later token only indicates that different branches selected the same discrete token at that position. Their underlying KV states, token histories, and conditional distributions may still differ substantially. In this sense, it is a token-level agreement event rather than a guarantee of KV-level convergence. This distinction is important: branches can agree on individual high-confidence tokens while still following different inference trajectories in KV-cache vector space.

\section{Method}
\label{sec:method}

Inspired by the characterizations in Sec.~\ref{sec:characterization}, we
observe that the main opportunity in block-wise dLLM inference is not to choose
a single globally optimal block size, but to exploit the complementary
trajectories induced by multiple block sizes during the same decoding process. We therefore propose \textbf{BlockBatch} (Fig.~\ref{fig:block_batching_overview}), a training-free framework that runs multiple block-size branches for a single request in one batched forward pass. Each branch alternates cheap local \emph{block denoising} with periodic \emph{full-sequence refreshes} that re-anchor them to a consistent KV state, while a confidence-gated \emph{merge} and leader-based \emph{sync} operations exchange compatible tokens and prune stale branches around bifurcation points.


\begin{figure*}[t]
  \centering
  \includegraphics[width=\linewidth]{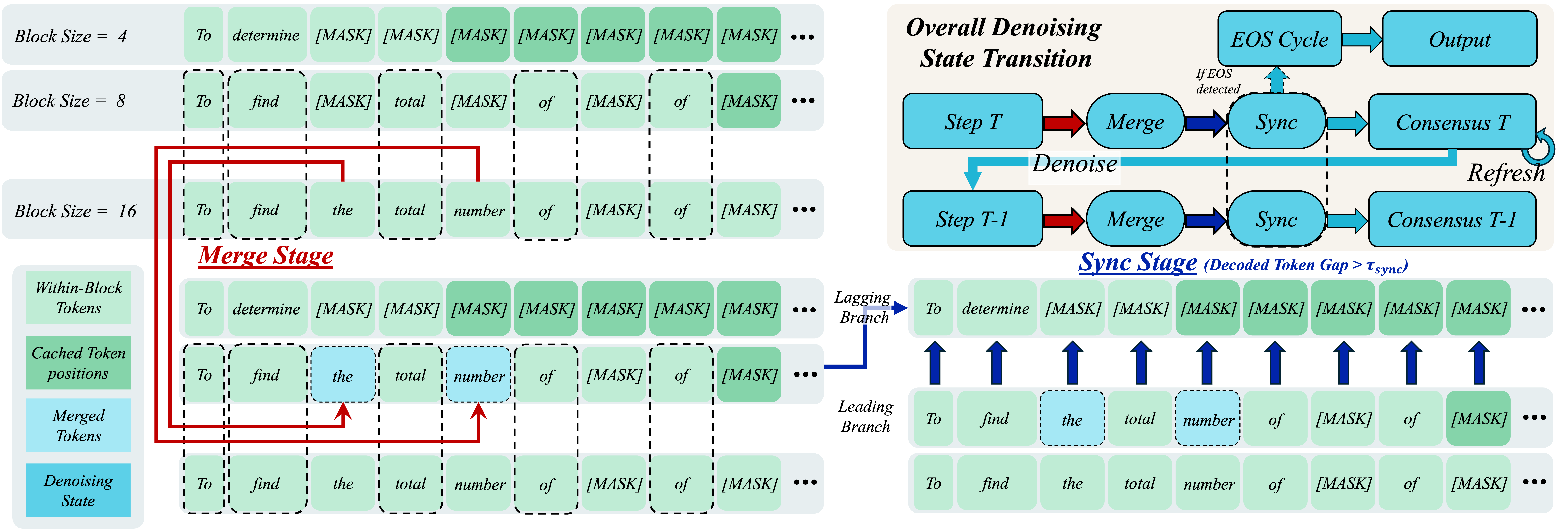}
  \vspace{-1em}
\caption{
\textbf{Overview of BlockBatch.}
The left and bottom parts illustrate confidence-gated merge and leader-based synchronization. The upper-right panel summarizes the global denoising
state transition: each step performs merge, sync, and periodic
refresh. If a branch predicts \texttt{EOS} before all preceding positions are decoded,
it enters an EOS cycle until the prefix before \texttt{EOS} is continuous, preventing
premature termination before final output.
}
  \label{fig:block_batching_overview}
\end{figure*}




\subsection{Exploiting Bifurcation}
\label{sec:bifur}
Block-wise denoising may enter low-progress regions where local updates do not
advance enough tokens. Instead of treating multiple branches only as redundant
work, we use their diversity as a source of useful proposals. Branches can share
tokens and KV states when their decoded prefixes are compatible. The detailed algorithm is shown in Algorithm~\ref{alg:merge-sync} (Appendix~\ref{app:block-batching-algorithm}), where the overview is shown in Fig.~\ref{fig:block_batching_overview}.

\paragraph{Confidence-Gated Merge}
The merge operation transfers high-confidence token proposals across compatible
branches. A source branch is compatible with a destination branch if all positions
decoded in both branches are exactly the same. For a masked position in the destination branch,
we inspect tokens proposed by compatible source branches and accept the best
candidate only when the destination branch's own probability map assigns that
token confidence above a fixed threshold. This gives lagging branches free
progress without forcing them to accept tokens inconsistent with their local
state.

\paragraph{Leader-Based Sync}
The sync operation handles larger progress gaps. If the leading branch has
decoded more than a threshold number of tokens beyond a lagging branch (this number is controlled by the sync threshold, see Appendix~\ref{app:threshold_ablations} for ablation Tab.~\ref{tab:sync_threshold}), the
lagging branch copies the leader's sequence state and KV cache row. Its block
window is then realigned to the first remaining masked position. Sync therefore
acts as a controlled reset: it preserves the fastest reliable trajectory while
preventing slow branches from wasting computation within stale regions of the KV
space.

\subsection{System-Level Optimization}
Fig.~\ref{fig:gpu_optimization}-(a) shows per-step block denoising latency on an H200 across varying batch size $\times$ block size combinations. Latency stays nearly flat in the memory-bound regime but rises sharply once the total token count exceeds roughly 256, where denoising becomes compute-bound. We exploit this observation so that multi-branch denoising costs about the same as denoising a single branch.

\begin{figure}[t]
  \centering
  \includegraphics[width=\linewidth]{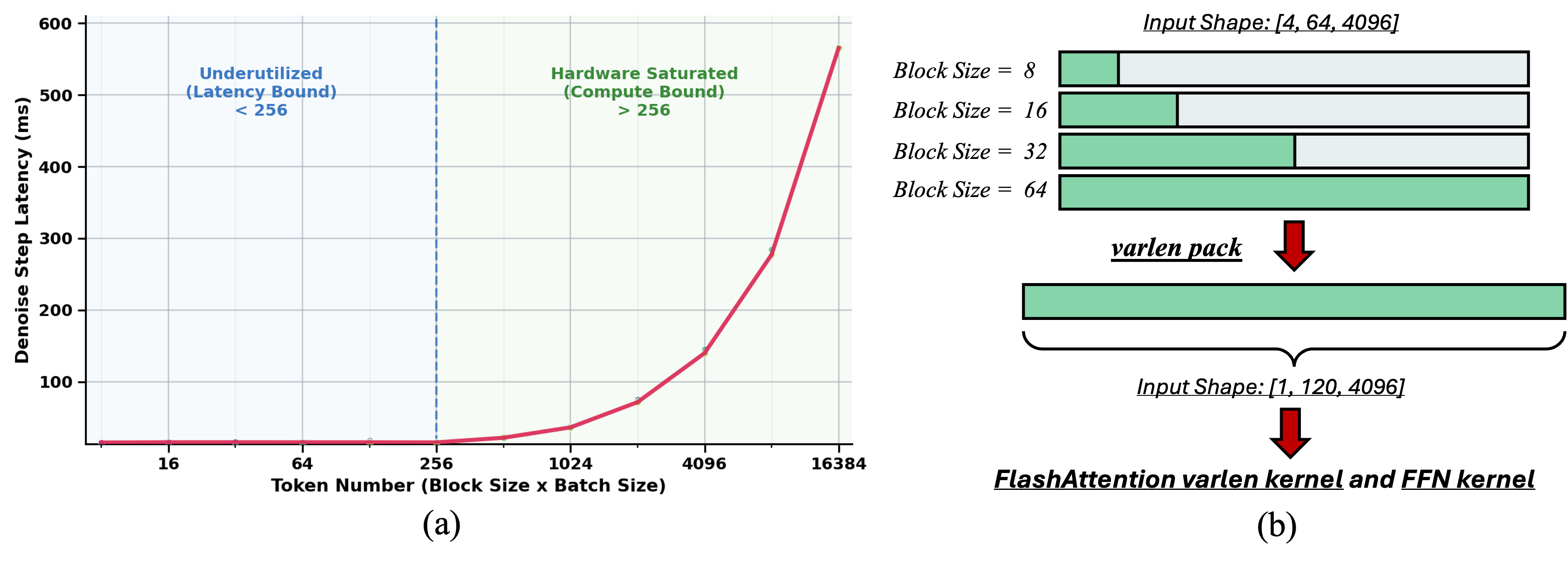}
  \vspace{-1em}
  \caption{
  \textbf{(a)} Per-step denoise latency on H200 vs.\ total token count.
  \textbf{(b)} \textbf{BlockBatch Denoise Step.} Unmasked query positions from all branches are packed into one variable-length buffer consumed by both varlen attention and the FFN, while a shared unified KV cache preserves per-branch attention semantics exactly.
 }
\label{fig:gpu_optimization}
\end{figure}

\paragraph{BlockBatch Denoise Optimization.}
At each denoise step, only the unmasked positions in each branch's current block produce queries. We pack these queries from all branches into a single variable-length buffer and issue one forward pass, supplying the cumulative offsets that FlashAttention's varlen kernel requires. The same packed tensor flows through every transformer block, so both attention and the FFN operate on this compact buffer end-to-end (Fig.~\ref{fig:gpu_optimization}-(b)).


\section{Experiments}
\label{sec:eval}
\subsection{Setup}
We evaluate BlockBatch on LLaDA-1.5-8B, LLaDA-Instruct-8B~\cite{nie2025large},
and Dream-Base-7B~\cite{ye2025dream}. We compare against Vanilla diffusion
decoding, Fast-dLLM dual cache~\cite{wu2026fastdllm}, and
LocalLeap~\cite{kong2025localleap}. Unless stated otherwise, all methods
generate up to 256 tokens with batch size 1 and confidence threshold 0.9. For
Fast-dLLM dual cache, we use block size 32. For LocalLeap, we use the default
script settings: anchor threshold 0.9, radius 4, and relaxed threshold 0.75 for
LLaDA and 0.8 for Dream. For BlockBatch, we run block sizes
$\{4,8,16,32,64,128\}$ in one fused generation pass. Each branch owns an
independent sequence row and KV-cache row; batching only changes execution
shape. We count one batched model forward as one NFE:
$
\mathrm{NFE}_{\mathrm{total}} = \mathrm{NFE}_{\mathrm{init}} + \mathrm{NFE}_{\mathrm{block}} + \mathrm{NFE}_{\mathrm{refresh}}.$
We evaluate on GSM8K~\cite{cobbe2021training}, MATH~\cite{hendrycksmath2021},
HumanEval~\cite{chen2021evaluating}, and MBPP~\cite{austin2021program}.
Accuracy is computed with the official task-specific evaluation logic; for code
tasks, generated code is postprocessed with the model-specific sanitization
pipeline before execution-based evaluation. Additional prompt, baseline, and
ablation settings are provided in Appendix~\ref{app:evaluation-setup}.

\subsection{Main Results}

Tab.~\ref{tab:block_batching_compact_b32} reports latency, NFE, and
accuracy across twelve model--task settings. The main result is that
BlockBatch is the most NFE-efficient method in every setting: it achieves
the lowest or tied-lowest NFE across all twelve comparisons, with
47.7--82.3\% fewer denoising NFEs than Vanilla and 1.7--5.8$\times$
end-to-end speedup. Beyond the Vanilla baseline, BlockBatch consistently
improves over Fast-dLLM, reducing NFE in all twelve settings by 26.6\%
on average and up to 33.6\% on LLaDA-Instruct-8B/MBPP. It is also faster
than or tied with Fast-dLLM in all settings, reaching up to 2.05$\times$
additional speedup on LLaDA-1.5-8B/HumanEval. Compared with LocalLeap,
BlockBatch obtains lower or tied NFE in all settings.

\begin{table*}[t]
\centering
\caption{
Benchmark results for Vanilla, Fast-dLLM Dual Cache, LocalLeap, and BlockBatch across LLaDA-1.5-8B, LLaDA-Instruct-8B, and Dream-Base-7B.
For Fast-dLLM Dual Cache, we report only the block size $32$ setting.
\emph{Notation:} Fast-dLLM denotes Fast-dLLM Dual Cache with block size $32$.
Green cells indicate BlockBatch results.
Subscripts on Latency show speedup over Vanilla; subscripts on NFE show reduction percentage relative to Vanilla.
}
\label{tab:block_batching_compact_b32}
\vspace{3pt}
\setlength{\tabcolsep}{3pt}
\renewcommand{\arraystretch}{1.15}
\small
\providecommand{\nfe}[1]{#1}
\newcommand{\spd}[1]{{\scriptsize\textcolor{bbtextgreen}{$\downarrow$#1}}}
\newcommand{\slow}[1]{{\scriptsize\textcolor{red}{$\uparrow$#1}}}
\resizebox{\textwidth}{!}{%
\begin{tabular}{llccccccccc}
\toprule
& 
& \multicolumn{3}{c}{\textbf{LLaDA-1.5-8B}} 
& \multicolumn{3}{c}{\textbf{LLaDA-Instruct-8B}} 
& \multicolumn{3}{c}{\textbf{Dream-Base-7B}} \\
\cmidrule(lr){3-5}
\cmidrule(lr){6-8}
\cmidrule(lr){9-11}
Task & Method 
& Lat. (s) $\downarrow$ & NFE $\downarrow$ & ACC (\%) $\uparrow$ 
& Lat. (s) $\downarrow$ & NFE $\downarrow$ & ACC (\%) $\uparrow$ 
& Lat. (s) $\downarrow$ & NFE $\downarrow$ & ACC (\%) $\uparrow$ \\
\midrule
\multirow{4}{*}{\makecell[l]{GSM8K}}
& Vanilla              
& 12.5 & 256.0 & 79.91
& 13.1 & 256.0 & 77.10 
& 10.4 & 256.0 & \textbf{75.13} \\
& \fdcell{Fast-dLLM}   
& \fdcell{4.9\spd{2.6$\times$}} & \fdcell{\nfe{84.7}\spd{67\%}} & \fdcell{79.68}
& \fdcell{2.7\spd{4.9$\times$}} & \fdcell{\nfe{86.4}\spd{66\%}} & \fdcell{\textbf{78.70}} 
& \fdcell{3.7\spd{2.8$\times$}} & \fdcell{\nfe{162.3}\spd{37\%}} & \fdcell{73.62} \\
& LocalLeap            
& 4.7\spd{2.7$\times$} & 72.5\spd{72\%} & \textbf{81.27}
& 2.3\spd{5.7$\times$} & 64.4\spd{75\%} & 78.16 
& 6.5\spd{1.6$\times$} & 142.9\spd{44\%} & 72.71 \\
& \bbcell{BlockBatch}  
& \bbcell{4.2\spd{3.0$\times$}} & \bbcell{\textbf{63.6}\spd{75\%}} & \bbcell{77.56}
& \bbcell{2.4\spd{5.5$\times$}} & \bbcell{\textbf{63.2}\spd{75\%}} & \bbcell{77.48} 
& \bbcell{3.7\spd{2.8$\times$}} & \bbcell{\textbf{133.8}\spd{48\%}} & \bbcell{72.48} \\
\midrule
\multirow{4}{*}{\makecell[l]{MATH}}
& Vanilla              
& 9.04 & 256.0 & \textbf{34.78}
& 9.7 & 256.0 & \textbf{40.14} 
& 8.1 & 256.0 & \textbf{40.10} \\
& \fdcell{Fast-dLLM}   
& \fdcell{5.5\spd{1.6$\times$}} & \fdcell{\nfe{104.4}\spd{59\%}} & \fdcell{33.20}
& \fdcell{3.3\spd{2.9$\times$}} & \fdcell{\nfe{107.8}\spd{58\%}} & \fdcell{37.12} 
& \fdcell{2.5\spd{3.2$\times$}} & \fdcell{\nfe{121.9}\spd{52\%}} & \fdcell{39.34} \\
& LocalLeap            
& 3.6\spd{2.5$\times$} & 80.3\spd{69\%} & 31.74
& 2.8\spd{3.5$\times$} & 85.3\spd{67\%} & 32.18 
& 3.3\spd{2.5$\times$} & \textbf{92.8}\spd{64\%} & 37.98 \\
& \bbcell{BlockBatch}  
& \bbcell{3.4\spd{2.7$\times$}} & \bbcell{\textbf{77.6}\spd{70\%}} & \bbcell{32.24}
& \bbcell{2.4\spd{4.0$\times$}} & \bbcell{\textbf{80.4}\spd{69\%}} & \bbcell{37.20} 
& \bbcell{2.2\spd{3.7$\times$}} & \bbcell{\textbf{92.8}\spd{64\%}} & \bbcell{39.44} \\
\midrule
\multirow{4}{*}{\makecell[l]{HumanEval}}
& Vanilla              
& 5.77 & 256.0 & \textbf{43.29}
& 5.7 & 256.0 & \textbf{40.24} 
& 4.7 & 256.0 & 50.00 \\
& \fdcell{Fast-dLLM}   
& \fdcell{4.3\spd{1.3$\times$}} & \fdcell{\nfe{88.7}\spd{65\%}} & \fdcell{39.02}
& \fdcell{2.6\spd{2.2$\times$}} & \fdcell{\nfe{90.4}\spd{65\%}} & \fdcell{36.59} 
& \fdcell{3.1\spd{1.5$\times$}} & \fdcell{\nfe{156.1}\spd{39\%}} & \fdcell{\textbf{52.44}} \\
& LocalLeap            
& 2.0\spd{2.9$\times$} & 75.8\spd{70\%} & 42.07
& 2.1\spd{2.7$\times$} & 69.8\spd{73\%} & 35.98 
& 2.6\spd{1.8$\times$} & 123.9\spd{52\%} & 48.78 \\
& \bbcell{BlockBatch}  
& \bbcell{2.1\spd{2.7$\times$}} & \bbcell{\textbf{63.0}\spd{75\%}} & \bbcell{39.63}
& \bbcell{2.0\spd{2.9$\times$}} & \bbcell{\textbf{64.6}\spd{75\%}} & \bbcell{38.41} 
& \bbcell{2.8\spd{1.7$\times$}} & \bbcell{\textbf{112.5}\spd{56\%}} & \bbcell{\textbf{52.44}} \\
\midrule
\multirow{4}{*}{\makecell[l]{MBPP}}
& Vanilla              
& 9.50 & 256.0 & \textbf{43.40}
& 9.9 & 256.0 & \textbf{41.40} 
& 8.0 & 256.0 & \textbf{55.80} \\
& \fdcell{Fast-dLLM}   
& \fdcell{4.2\spd{2.3$\times$}} & \fdcell{\nfe{76.1}\spd{70\%}} & \fdcell{38.00}
& \fdcell{2.2\spd{4.5$\times$}} & \fdcell{\nfe{68.2}\spd{73\%}} & \fdcell{38.40} 
& \fdcell{2.6\spd{3.1$\times$}} & \fdcell{\nfe{111.9}\spd{56\%}} & \fdcell{53.20} \\
& LocalLeap            
& 2.7\spd{3.5$\times$} & 58.0\spd{77\%} & 40.20
& 1.9\spd{5.2$\times$} & 56.1\spd{78\%} & \textbf{41.40} 
& 2.9\spd{2.8$\times$} & 82.8\spd{68\%} & 54.60 \\
& \bbcell{BlockBatch}  
& \bbcell{2.6\spd{3.7$\times$}} & \bbcell{\textbf{54.5}\spd{79\%}} & \bbcell{40.00}
& \bbcell{1.7\spd{5.8$\times$}} & \bbcell{\textbf{45.3}\spd{82\%}} & \bbcell{39.60} 
& \bbcell{2.3\spd{3.5$\times$}} & \bbcell{\textbf{81.1}\spd{68\%}} & \bbcell{52.00} \\
\bottomrule
\end{tabular}%
}
\end{table*}

\subsection{Ablation Studies}
\label{subsec:ablation}

BlockBatch depends on three main design choices: the synchronization threshold, the refresh interval, and the set of block sizes used for parallel exploration. These parameters control the central tradeoff of the method: allowing enough branch diversity to discover efficient trajectories while limiting computational overhead.

\paragraph{Synchronization Threshold.}
The synchronization threshold determines how far the leading branch is allowed to advance before a lagging branch is synchronized to it (See Appendix~\ref{app:threshold_ablations}, Tab.~\ref{tab:sync_threshold}). This threshold is important because it controls the balance between exploration and stability. If synchronization is triggered too aggressively, branches are forced to follow the current leader too early, falling back to a greedy regime that does not guarantee a globally optimal trajectory.

On the other hand, if the threshold is too large, lagging branches can remain trapped in slow or stale block-denoising regions. In the KV-space view, these branches may continue accumulating local cache drift without contributing useful progress. Synchronization therefore acts as a branch-space contraction mechanism: it preserves useful exploration within a bounded window while preventing slow branches from wasting computation after they have fallen too far behind. The ablation results show that the best performance is obtained at an intermediate threshold, where branch diversity is maintained but unproductive drift is controlled.

\paragraph{Refresh Interval.}
The refresh interval $R$ controls how often a full-sequence refresh is applied. (See appendix \ref{app:threshold_ablations}, table \ref{tab:refresh_interval}) From the refresh contraction bound, stability requires $\beta(1+\lambda_b)^R < 1$, which implies $R < \frac{\log(1/\beta)}{\log(1+\lambda_b)}$. Thus, $R$ cannot be too large: if refresh is too infrequent, accumulated block-denoise error can dominate the contraction effect of full-sequence refresh. However, $R$ also cannot be too small in practice, because frequent full-sequence refreshes increase model-call overhead and end-to-end latency. Therefore, the refresh interval must lie in a feasible window: small enough to control KV-cache drift, but large enough to amortize the cost of global recomputation. The ablation results in Tab.~\ref{tab:refresh_interval} confirm this behavior, showing that the optimal refresh setting is neither overly frequent nor overly sparse.

\paragraph{Block-Size Configuration.}
The choice of block sizes determines the diversity of decoding trajectories explored by BlockBatch (Shown in Fig.~\ref{fig:block_combo_llada_humaneval}). Using multiple block sizes allows the method to exploit idle GPU capacity by evaluating several denoising schedules in parallel. Smaller blocks provide conservative, locally stable progress, while larger blocks can decode more aggressively when the model is confident. Combining them increases the probability that at least one branch follows an efficient trajectory within the synchronization window.

\begin{figure}[t]
    \centering
    \includegraphics[width=1\linewidth]{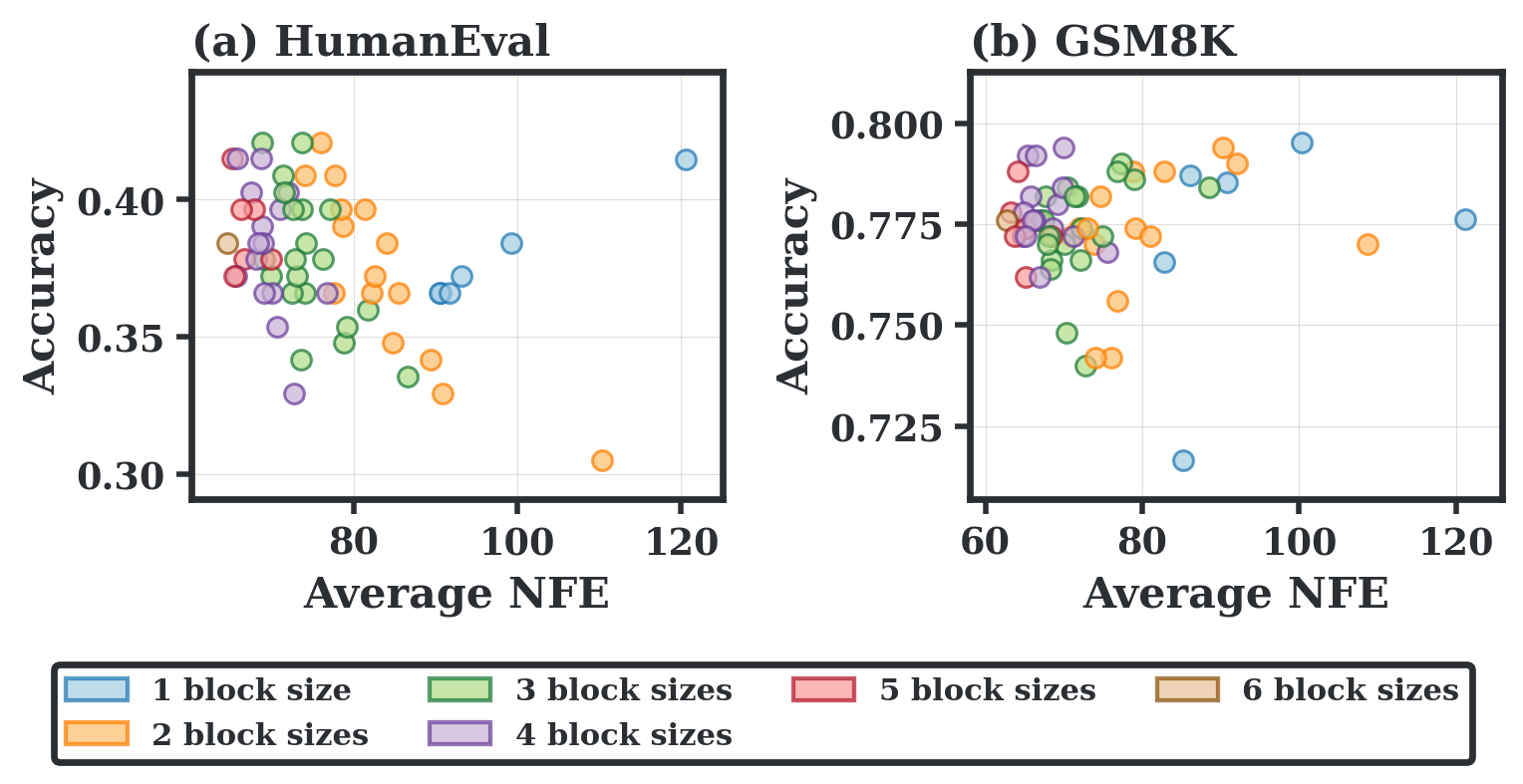}
    \vspace{-0.5em}
    \caption{Each point represents one LLaDA BlockBatch candidate block-size set \(S \subseteq \mathcal{B}\), where \(\mathcal{B}=\{4,8,16,32,64,128\}\). Colors encode the set cardinality \(|S|\). Thus, the singleton group contains \(\binom{6}{1}=6\) fixed block-size configurations, the two-block-size group contains \(\binom{6}{2}=15\) pairwise configurations, and in general the \(k\)-block-size group contains \(\binom{6}{k}\) candidate sets.}
    \label{fig:block_combo_llada_humaneval}
\end{figure}

However, adding more branches is not free. Although more block sizes can reduce NFE by increasing exploration and enabling useful synchronization, each additional branch also increases batching, cache-management, and refresh overhead. Therefore, the best block-size configuration must balance NFE reduction against end-to-end latency. The ablation results in Fig.~\ref{fig:block_combo_llada_humaneval} show that multi-scale block-size sets are more effective than narrow configurations, but also that the benefit saturates once the additional branches no longer provide enough useful trajectory diversity to justify their overhead.
\section{Conclusion}

We presented BlockBatch, a training-free framework that treats block size as a branching axis for dLLM inference. By coordinating multiple block-size branches through confidence-gated merging, leader-based synchronization, and periodic full-sequence refresh within a single fused forward pass, BlockBatch consistently reduces NFEs and latency on LLaDA and Dream across four benchmarks while preserving accuracy over state-of-the-art dLLM inference frameworks. Block-size diversity is a practical, underexplored axis for branch-parallel diffusion decoding.

\section{Limitations}
\label{section:limitations}

BlockBatch is designed as a training-free inference-time acceleration method,
and its effectiveness depends on a decoding regime in which block-size diversity
produces compatible but non-identical trajectories. In low-diversity regimes,
such as near-deterministic or extremely easy prompts where all selected block
sizes converge to the same decoded sequence with little or no bifurcation,
additional branches contribute limited new information while still introducing
coordination overhead. Conversely, BlockBatch also assumes that the selected
block sizes are individually stable enough to provide useful candidate
trajectories. If a model becomes unreliable under large block sizes, such
branches may reduce accuracy and should be adaptively deselected during
inference rather than retained throughout the full decoding process.







\bibliography{custom}

\appendix

\section*{Appendix}
\begin{itemize}
\item Sec.~\ref{app:threshold_ablations}: Ablation Tables
\item Sec.~\ref{app:token-level}: Token level characterization
\begin{itemize}
\item \hyperref[app:later-stage-consensus]{Later-Stage Consensus Token Analysis}
\item \hyperref[app:tokenbifurcationCaseStudies]{Token Bifurcation Case studies}
\end{itemize}
\item Sec.~\ref{app:kv-cache-vector-space}: KV level characterization: proofs and more case studies
\begin{itemize}
\item \hyperref[app:basic-operations]{Basic Operations}
\item \hyperref[prop:moreUpdateFromFullRefresh]{Proposition 1: Block Denoise vs.\ Full Refresh}
\item \hyperref[refresh_stablize]{Proposition 2: Periodic Refresh Stabilizes Block Denoising}
\item \hyperref[prop:Branch_bifurcation_tanspace]{Proposition 3: Branch Bifurcation as Tangent-Space Separation}
\item \hyperref[app:kv-cache-case-study]{Case Study: KV Cache Vector Space Properties}
\end{itemize}
\item Sec.~\ref{app:block-batching-algorithm}: Block Batching Algorithm
\begin{itemize}
\item \hyperref[app:fused-block-batching]{Fused Block Batching with a Unified KV Cache}
\item \hyperref[app:merge-sync-policy]{Merge and Synchronization Policy}
\end{itemize}
\item Sec.~\ref{app:reproducibility}: Reproducibility, Artifacts and Computer Details
\begin{itemize}
\item \hyperref[app:scientific-artifacts]{Scientific Artifacts}
\item \hyperref[app:artifact-licenses]{Artifact Licenses and Intended Use}
\item \hyperref[app:evaluation-setup]{Evaluation Setup}
\item \hyperref[app:compute-budget]{Computational Infrastructure and Budget}
\end{itemize}
\end{itemize}

\section{Appendix: Ablation Tables}
\label{app:threshold_ablations}

\begin{table}[H]
\centering
\caption{
Sync-threshold $\tau$ ablation.
\textbf{Bold} indicates the best value within each column.
}
\vspace{-1em}
\label{tab:sync_threshold}
\scriptsize
\setlength{\tabcolsep}{4pt}
\renewcommand{\arraystretch}{1.15}
\resizebox{\linewidth}{!}{
\begin{tabular}{l cc cc cc cc}
\toprule
& \multicolumn{4}{c}{LLaDA 8B-Inst.} & \multicolumn{4}{c}{Dream Base-7B} \\
\cmidrule(lr){2-5} \cmidrule(lr){6-9}
& \multicolumn{2}{c}{GSM8K} & \multicolumn{2}{c}{HumanEval}
& \multicolumn{2}{c}{GSM8K} & \multicolumn{2}{c}{HumanEval} \\
\cmidrule(lr){2-3} \cmidrule(lr){4-5} \cmidrule(lr){6-7} \cmidrule(lr){8-9}
$\tau$ & Acc. & NFE & Acc. & NFE & Acc. & NFE & Acc. & NFE \\
\midrule
4   & 0.773 & \textbf{67.6} & 0.354 & 69.6 & 0.726 & \textbf{138.2} & \textbf{0.494} & 122.7 \\
8   & 0.770 & 68.6 & 0.378 & 69.9 & \textbf{0.732} & 138.7 & \textbf{0.494} & 120.5 \\
16  & \textbf{0.784} & 69.9 & 0.384 & \textbf{69.4} & 0.725 & 139.4 & 0.470 & \textbf{114.6} \\
32  & 0.780 & 70.6 & \textbf{0.415} & 70.5 & 0.719 & 140.6 & 0.476 & 117.3 \\
64  & 0.780 & 73.1 & 0.384 & 71.6 & 0.703 & 143.9 & 0.427 & 116.7 \\
\bottomrule
\end{tabular}
}
\end{table}

\begin{table}[H]
\centering
\caption{
Refresh-interval $R$ ablation.
\textbf{Bold} indicates the best value within each column.
}
\vspace{-1em}
\label{tab:refresh_interval}
\scriptsize
\setlength{\tabcolsep}{4pt}
\renewcommand{\arraystretch}{1.15}
\resizebox{\linewidth}{!}{
\begin{tabular}{l cc cc cc cc}
\toprule
& \multicolumn{4}{c}{LLaDA 8B-Inst.} & \multicolumn{4}{c}{Dream Base-7B} \\
\cmidrule(lr){2-5} \cmidrule(lr){6-9}
& \multicolumn{2}{c}{GSM8K} & \multicolumn{2}{c}{HumanEval}
& \multicolumn{2}{c}{GSM8K} & \multicolumn{2}{c}{HumanEval} \\
\cmidrule(lr){2-3} \cmidrule(lr){4-5} \cmidrule(lr){6-7} \cmidrule(lr){8-9}
$R$ & Acc. & NFE & Acc. & NFE & Acc. & NFE & Acc. & NFE \\
\midrule
4   & 0.778 & \textbf{59.0} & 0.372 & \textbf{56.8} & \textbf{0.780} & 143.0 & \textbf{0.543} & 120.8 \\
8   & 0.784 & 59.9 & 0.378 & 58.9 & \textbf{0.780} & 138.9 & 0.512 & \textbf{110.1} \\
16  & 0.770 & 61.1 & \textbf{0.402} & 62.4 & 0.752 & 137.5 & 0.476 & 111.8 \\
32  & 0.776 & 62.7 & 0.384 & 64.6 & 0.730 & 134.7 & 0.524 & 112.5 \\
64  & \textbf{0.788} & 64.6 & 0.348 & 67.1 & 0.744 & \textbf{134.2} & 0.494 & 112.8 \\
128 & 0.752 & 65.6 & 0.354 & 71.2 & 0.746 & 135.2 & 0.500 & 117.2 \\
\bottomrule
\end{tabular}
}
\end{table}

\section{Appendix: Token-Level Characterization}
\label{app:token-level}

\subsection{Later-Stage Consensus and Bifurcation Tokens}
\label{app:later-stage-consensus}

We further analyze token-level branch behavior by separating two phenomena:
\emph{later-stage consensus}, where previously diverged branches produce the
same token again at a later position, and \emph{bifurcation}, where two
block-size trajectories first commit different tokens after a shared prefix.
These two phenomena play different roles. Later-stage consensus indicates local
token agreement, but does not imply that the full trajectories have re-converged.
Bifurcation tokens, by contrast, mark the positions at which block-size branches
begin to follow distinct conditional generation paths.

Fig.~\ref{fig:category-humaneval-agreement} and
Fig.~\ref{fig:category-gsm8k-agreement} characterize generated token positions
by linguistic category---digit, operator, whitespace, word, and other---and by
cross-block-size agreement. Each stacked bar reports the average number of
positions per sequence in that category, while the shading indicates how many of
the six block-size configurations produced the same token. The overlaid curve
shows the mean agreement score for each category.

A clear domain-specific pattern appears. In HumanEval
(Fig.~\ref{fig:category-humaneval-agreement}), whitespace and operator tokens
achieve the highest agreement and are dominated by $6/6$ agreement, reflecting
the syntactic regularity of code generation. In GSM8K
(Fig.~\ref{fig:category-gsm8k-agreement}), digit tokens are more frequent
(avg.\ $61.2$ positions per sequence, compared with $23.7$ in HumanEval) and
also show high agreement, consistent with the repeated numerical structure of
arithmetic reasoning traces. Across both domains, the highest-agreement
categories are mostly low-semantic-weight tokens: they control formatting,
syntax, or local surface structure rather than the main reasoning content.

\begin{figure}[H]
\centering
\includegraphics[width=0.82\linewidth]{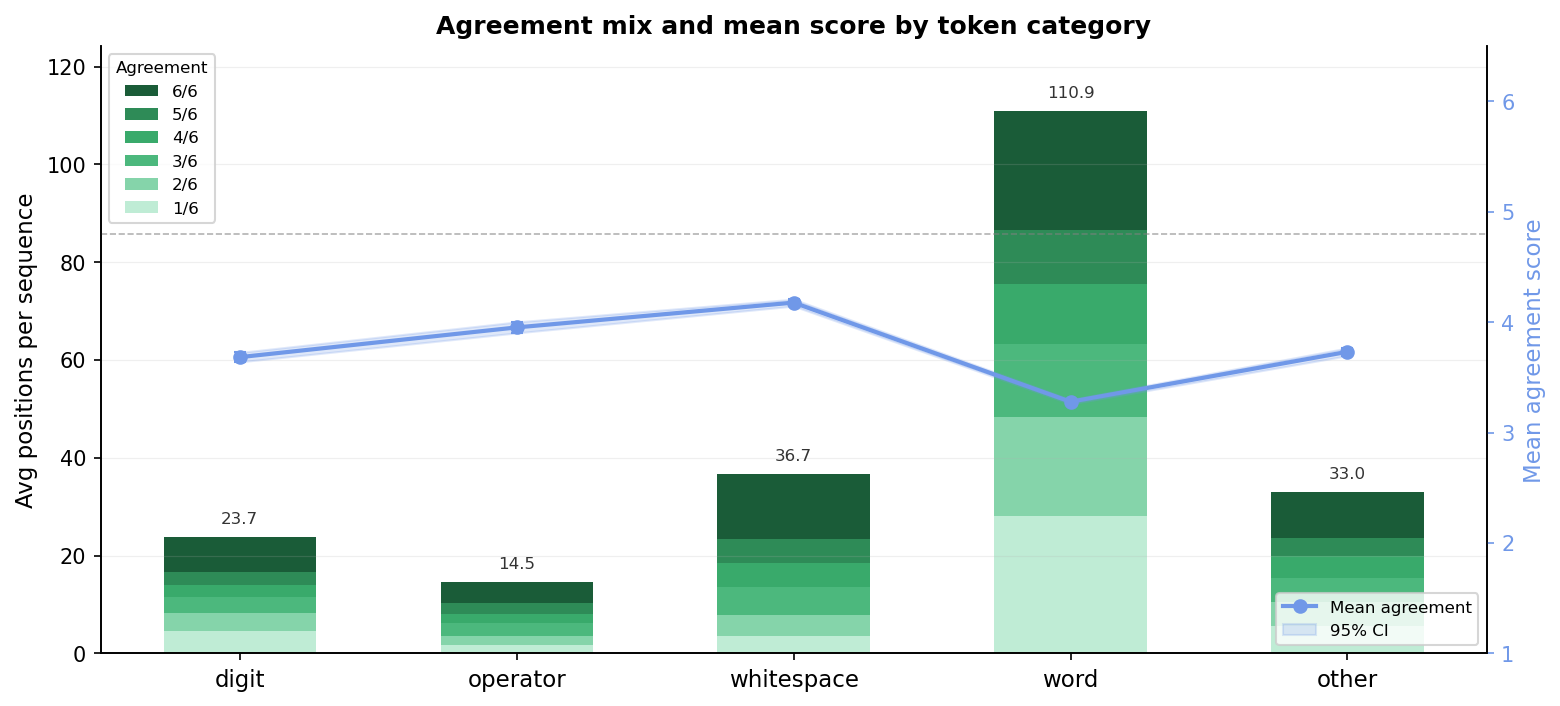}
\vspace{-0.4em}
\caption{
\textbf{Token-category agreement profile for HumanEval.}
Each stacked bar shows the average number of generated-token positions in a
category, shaded by the fraction of block-size configurations that produce the
same token at that position ($6/6$ = darkest green; $1/6$ = lightest). The blue
curve gives the mean agreement score per category. Whitespace and operator
tokens dominate later-stage consensus, reflecting the syntactic regularity of
code generation.
}
\label{fig:category-humaneval-agreement}
\vspace{-0.6em}
\end{figure}
\FloatBarrier

\begin{figure}[H]
\centering
\includegraphics[width=0.82\linewidth]{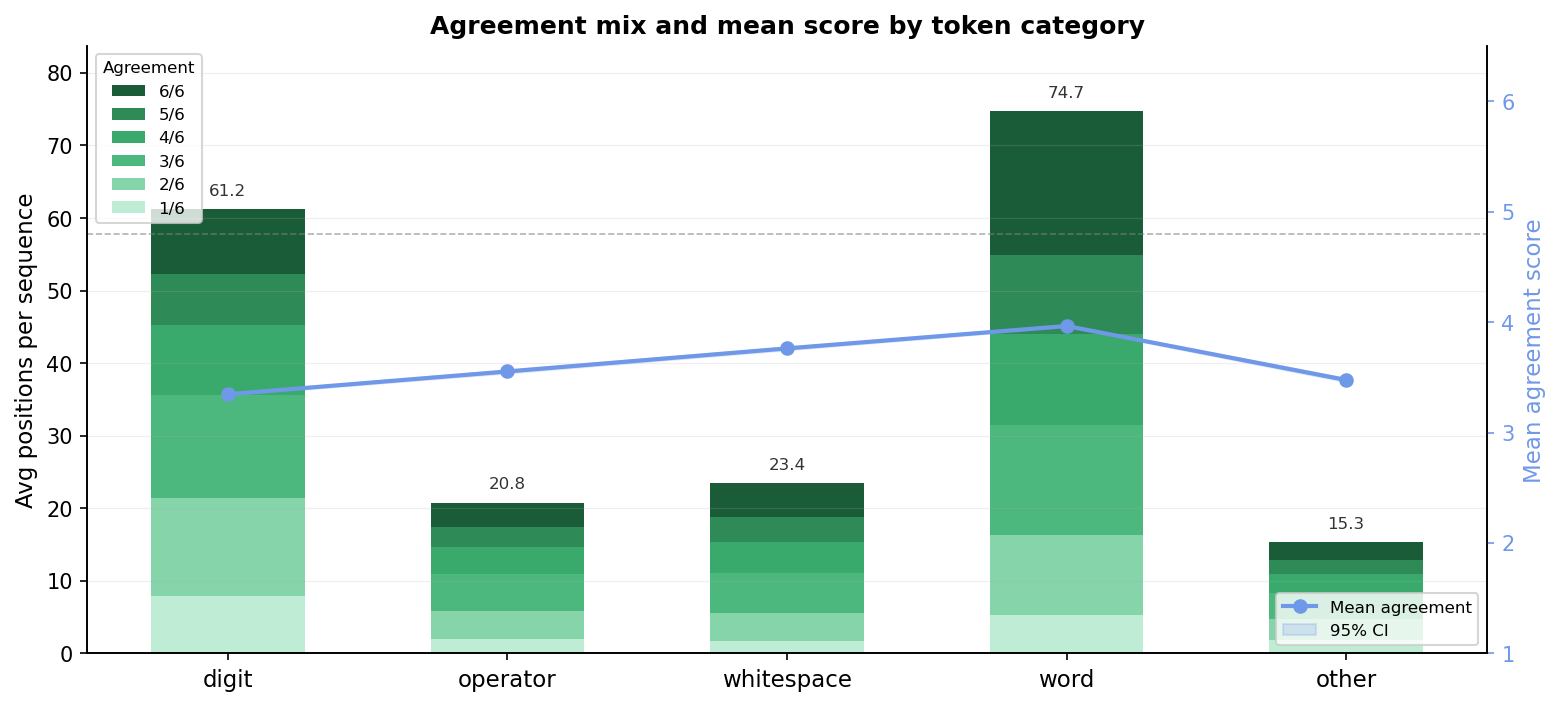}
\vspace{-0.4em}
\caption{
\textbf{Token-category agreement profile for GSM8K.}
Digit tokens are substantially more prevalent than in HumanEval
(avg.\ $61.2$ positions per sequence) and exhibit high mean agreement,
consistent with the repetitive numerical structure of arithmetic reasoning
chains. Word tokens are the most frequent category overall
(avg.\ $74.7$ positions per sequence) but show lower per-position agreement,
indicating that they more commonly mark bifurcation points where block-size
trajectories diverge.
}
\label{fig:category-gsm8k-agreement}
\vspace{-0.6em}
\end{figure}
\FloatBarrier

This category-level pattern explains why later-stage consensus is useful but
limited. Although branches may re-agree on individual tokens after an earlier
split, those tokens are often formatting or locally constrained tokens rather
than semantically decisive reasoning steps. To test whether such tokens can
meaningfully steer generation, we perform a seeded-consensus experiment. For
each sample, we identify token positions with strict $6/6$ cross-block-size
agreement after the initial shared prefix and pre-fill those positions before
normal denoising begins.

\begin{table}[H]
\centering
\scriptsize
\setlength{\tabcolsep}{3.2pt}
\renewcommand{\arraystretch}{1.03}
\caption{
\textbf{Effect of seeding later-stage consensus tokens on Dream-v0-Instruct-7B.}
\emph{Seeded} pre-fills later-stage consensus positions before denoising, where
consensus is defined as $6/6$ cross-block-size agreement after the initial
shared prefix. On average, this seeds $4.6$ tokens per GSM8K sample and $3.5$
tokens per HumanEval sample. $\Delta$Acc. is Seeded minus Baseline; $\Delta$NFE
is Baseline minus Seeded, so positive $\Delta$NFE indicates reduced denoising
cost.
}
\label{tab:seeded-consensus}
\resizebox{\linewidth}{!}{%
\begin{tabular}{llrrrrrr}
\toprule
\multirow{2}{*}{Task}
& \multirow{2}{*}{Block size}
& \multicolumn{3}{c}{Accuracy (\%)}
& \multicolumn{3}{c}{Average NFE} \\
\cmidrule(lr){3-5}
\cmidrule(lr){6-8}
& & Baseline & Seeded & $\Delta$Acc.
  & Baseline & Seeded & $\Delta$NFE \\
\midrule
GSM8K & 4   & 76.04 & 74.75 & $-1.29$ & 154.2 & 152.9 & 1.3 \\
GSM8K & 8   & 76.04 & 74.91 & $-1.13$ & 104.3 & 102.3 & 2.0 \\
GSM8K & 16  & 74.98 & 73.62 & $-1.36$ &  80.0 &  77.5 & 2.5 \\
GSM8K & 32  & 73.69 & 73.24 & $-0.45$ &  67.9 &  65.4 & 2.5 \\
GSM8K & 64  & 70.58 & 70.13 & $-0.45$ &  64.8 &  61.6 & 3.2 \\
GSM8K & 128 & 41.77 & 42.61 & $+0.84$ &  81.8 &  79.0 & 2.8 \\
\midrule
HumanEval & 4   & 54.88 & 50.00 & $-4.88$ & 174.3 & 173.7 & 0.6 \\
HumanEval & 8   & 55.49 & 54.27 & $-1.22$ & 129.0 & 125.3 & 3.7 \\
HumanEval & 16  & 55.49 & 55.49 & $+0.00$ & 102.6 & 100.1 & 2.5 \\
HumanEval & 32  & 55.49 & 56.71 & $+1.22$ &  86.4 &  85.8 & 0.6 \\
HumanEval & 64  & 53.05 & 52.44 & $-0.61$ &  82.7 &  81.9 & 0.8 \\
HumanEval & 128 & 33.54 & 34.15 & $+0.61$ &  67.0 &  64.1 & 2.9 \\
\bottomrule
\end{tabular}%
}
\vspace{-0.6em}
\end{table}
\FloatBarrier

Table~\ref{tab:seeded-consensus} shows that seeding later-stage consensus
tokens produces only small NFE reductions and does not systematically improve
accuracy. On average, only $4.6$ GSM8K tokens and $3.5$ HumanEval tokens per
sample satisfy the strict post-prefix $6/6$ agreement criterion. This supports
the view that later-stage consensus tokens are locally reusable but usually not
strong enough to determine the global denoising trajectory. Therefore, the
semantically important variation between branches is better captured by where
they first bifurcate, rather than by the small number of tokens on which they
later re-agree.

We therefore measure the average bifurcation length between every pair of
single-block-size generations. For a pair of block sizes, bifurcation length is
the length of the common generated-token prefix before the two outputs first
differ. Smaller values indicate earlier trajectory separation.

\begin{table}[H]
\centering
\tiny
\setlength{\tabcolsep}{1.8pt}
\renewcommand{\arraystretch}{1.0}
\caption{
\textbf{Average bifurcation length between single-block-size generations.}
For each pair of block sizes, bifurcation length is defined as the length of the
common generated-token prefix before the two outputs first differ, following
\texttt{analyze\_bifurcation.py}. When the two outputs have different lengths,
the shorter generated sequence is used. Lower values indicate earlier
trajectory divergence. Bold entries denote the earliest bifurcation within each
model--dataset row.
}
\label{tab:bifurcate-length}
\resizebox{\linewidth}{!}{%
\begin{tabular}{llrrrrrrrrrrrrrrr}
\toprule
\multirow{2}{*}{Model}
& \multirow{2}{*}{Dataset}
& \multicolumn{15}{c}{Block-size pair} \\
\cmidrule(lr){3-17}
& & 4--8 & 4--16 & 4--32 & 4--64 & 4--128
& 8--16 & 8--32 & 8--64 & 8--128
& 16--32 & 16--64 & 16--128
& 32--64 & 32--128 & 64--128 \\
\midrule
Dream & GSM8K
& 23.4 & 18.0 & 16.0 & 15.9 & \textbf{15.6}
& 25.5 & 20.4 & 19.0 & 18.5
& 27.1 & 23.3 & 22.1
& 32.1 & 29.8 & 44.3 \\
Dream & HumanEval
& 34.9 & 22.0 & 17.9 & 17.8 & \textbf{16.4}
& 37.1 & 20.9 & 22.6 & 19.0
& 28.1 & 23.3 & 23.1
& 34.5 & 29.6 & 47.2 \\
Dream & MATH
& 32.6 & 23.4 & 22.0 & \textbf{21.8} & 22.0
& 34.1 & 27.5 & 26.5 & 26.2
& 40.3 & 34.5 & 33.4
& 47.5 & 43.1 & 59.9 \\
Dream & MBPP
& 33.3 & 21.6 & \textbf{20.9} & 21.7 & 21.7
& 32.6 & 27.0 & 25.6 & 26.9
& 36.7 & 33.1 & 29.0
& 53.8 & 43.5 & 56.7 \\
\midrule
LLaDA & GSM8K
& 35.2 & 24.8 & 22.0 & 21.2 & \textbf{20.6}
& 32.7 & 26.4 & 25.0 & 23.4
& 34.2 & 28.0 & 25.8
& 38.9 & 33.0 & 48.7 \\
LLaDA & HumanEval
& 64.5 & 48.5 & 42.6 & 31.3 & 32.5
& 52.2 & 39.0 & \textbf{30.0} & 30.5
& 53.7 & 38.0 & 37.4
& 52.6 & 44.0 & 54.3 \\
LLaDA & MATH
& 34.0 & 24.7 & 22.4 & 22.2 & \textbf{22.2}
& 34.5 & 28.0 & 25.4 & 25.5
& 38.0 & 30.6 & 29.7
& 43.4 & 39.0 & 54.7 \\
LLaDA & MBPP
& 31.6 & 21.9 & \textbf{17.9} & 19.1 & 18.9
& 28.3 & 20.8 & 19.5 & 21.1
& 26.7 & 22.3 & 23.3
& 26.2 & 25.5 & 37.7 \\
\bottomrule
\end{tabular}%
}
\vspace{-0.6em}
\end{table}
\FloatBarrier

Table~\ref{tab:bifurcate-length} shows that substantially different block-size
pairs generally bifurcate earlier than nearby block-size pairs. For example,
pairs involving small and large blocks, such as $4$--$128$ or $8$--$128$, often
split earlier than adjacent pairs such as $32$--$64$ or $64$--$128$. This pattern
suggests that block-size diversity is not merely redundant decoding. Different
block sizes expose distinct token-level trajectories, while still often
preserving enough shared structure to support synchronization and selective
merging.

The exact bifurcation position depends on both the dataset and the model.
Code-generation tasks often share long syntactic prefixes before diverging in
implementation details, while arithmetic reasoning tasks can bifurcate through
small changes in wording, intermediate computation, or final-answer formatting.
Table~\ref{tab:gsm8k-token-bifurcation-sample-1}, \ref{tab:gsm8k-token-bifurcation-sample-2}, \ref{tab:gsm8k-token-bifurcation-sample-3} provides case studies using raw output spans. Rather than summarizing each
generation abstractly, the table shows the first visible split after the shared
reasoning prefix and the final answer span for each inference regime.

\FloatBarrier

\subsection{Token Bifurcation Case studies}
\label{app:tokenbifurcationCaseStudies}

This subsection reports three \texttt{GSM8K} examples from
\texttt{LLaDA-8B-Instruct}. For each example, we show the first raw output
bifurcation among the six single block-size branches and the terminal answer
span generated by each method. The bifurcation position is computed after
whitespace normalization for readability. The reported continuation is the
raw branch-local text beginning at the first divergent token.


\newcolumntype{Y}{>{\raggedright\arraybackslash}X}
\newcommand{\rawfrag}[1]{{\ttfamily\scriptsize #1}}
\newcommand{\Yes}{\textsc{Y}}
\newcommand{\No}{\textsc{N}}
\newcommand{\Invalid}{\texttt{invalid}}

\paragraph{\texttt{GSM8K} sample~1.}
Target answer: \(18\). The six single block-size branches share the prefix
\rawfrag{... total number of eggs laid} and first bifurcate at token position~26.

\begin{table}[H]
\centering
\scriptsize
\setlength{\tabcolsep}{2.5pt}
\renewcommand{\arraystretch}{1.08}
\caption{\textbf{Token bifurcation on \texttt{GSM8K} sample~1.}
Several branches compute the correct value \(18\), but strict exact match
depends on whether the terminal span contains the required
\texttt{\#\#\#\#} answer marker.}
\label{tab:gsm8k-token-bifurcation-sample-1}
\begin{tabularx}{\linewidth}{@{}lYYccr@{}}
\toprule
Method & First divergent continuation & Terminal answer span
& Pred. & EM & NFE \\
\midrule
Vanilla
& \rawfrag{by Janet's ducks per day}
& \rawfrag{farmers' market. \#\#\#\# 18}
& \(18\) & \Yes & 256 \\

BS=4
& \rawfrag{by Janet's ducks per day}
& \rawfrag{The final answer is: \textbackslash{}boxed\{18\}}
& \Invalid & \No & 154 \\

BS=8
& \rawfrag{by the ducks per day}
& \rawfrag{farmers' market. \#\#\#\# 18}
& \(18\) & \Yes & 120 \\

BS=16
& \rawfrag{by Janet's ducks per day}
& \rawfrag{farmers' market. \#\#\#\# 18}
& \(18\) & \Yes & 121 \\

BS=32
& \rawfrag{by Janet's ducks per day}
& \rawfrag{The final answer is: \textbackslash{}boxed\{18\}}
& \Invalid & \No & 123 \\

BS=64
& \rawfrag{per day. 2. Subtract}
& \rawfrag{farmers' market. \#\#\#\# 18}
& \(18\) & \Yes & 111 \\

BS=128
& \rawfrag{by Janet's ducks per day}
& \rawfrag{Janet makes \$18 every day at the farmers' market.}
& \Invalid & \No & 105 \\

BlockBatch
& \rawfrag{by Janet's ducks per day}
& \rawfrag{farmers' market. \#\#\#\# 24}
& \(24\) & \No & 557 \\
\bottomrule
\end{tabularx}
\end{table}
\FloatBarrier

\paragraph{\texttt{GSM8K} sample~2.}
Target answer: \(3\). The six single block-size branches share the prefix
\rawfrag{... robe takes 2 bolts of blue fiber and half that} and first
bifurcate at token position~11.

\begin{table}[H]
\centering
\scriptsize
\setlength{\tabcolsep}{2.5pt}
\renewcommand{\arraystretch}{1.08}
\caption{\textbf{Token bifurcation on \texttt{GSM8K} sample~2.}
Most branches preserve the correct arithmetic trajectory. Failures are mainly
strict-format failures caused by missing terminal \texttt{\#\#\#\#} extraction.}
\label{tab:gsm8k-token-bifurcation-sample-2}
\begin{tabularx}{\linewidth}{@{}lYYccr@{}}
\toprule
Method & First divergent continuation & Terminal answer span
& Pred. & EM & NFE \\
\midrule
Vanilla
& \rawfrag{much white fiber. To find}
& \rawfrag{total of 3 bolts. \#\#\#\# 3}
& \(3\) & \Yes & 256 \\

BS=4
& \rawfrag{much white fiber. First, calculate}
& \rawfrag{it takes is 3. \#\#\#\# 3}
& \(3\) & \Yes & 110 \\

BS=8
& \rawfrag{much white fiber. So, the}
& \rawfrag{= 3 \textbackslash{}text\{ bolts\} \textbackslash{}] \#\#\#\# 3}
& \(3\) & \Yes & 78 \\

BS=16
& \rawfrag{amount of white fiber. So,}
& \rawfrag{= 3 \textbackslash{}text\{ bolts\} \textbackslash{}] \#\#\#\# 3}
& \(3\) & \Yes & 66 \\

BS=32
& \rawfrag{amount of white fiber. First,}
& \rawfrag{robe takes a total of 3 bolts.}
& \Invalid & \No & 76 \\

BS=64
& \rawfrag{amount of white fiber. First,}
& \rawfrag{2 bolts + 1 bolt = 3 bolts. \#\#\#\# 3}
& \(3\) & \Yes & 87 \\

BS=128
& \rawfrag{amount of white fiber.}
& \rawfrag{blue fiber and half that amount of white fiber.}
& \Invalid & \No & 77 \\

BlockBatch
& \rawfrag{amount of white fiber. Half}
& \rawfrag{white fiber = 3 bolts. \#\#\#\# 3}
& \(3\) & \Yes & 192 \\
\bottomrule
\end{tabularx}
\end{table}
\FloatBarrier

\paragraph{\texttt{GSM8K} sample~3.}
Target answer: \(70{,}000\). The six single block-size branches share the
prefix \rawfrag{... need to follow these steps: 1. Calculate the} and first
bifurcate at token position~13.

\begin{table}[H]
\centering
\scriptsize
\setlength{\tabcolsep}{2.5pt}
\renewcommand{\arraystretch}{1.08}
\caption{\textbf{Token bifurcation on \texttt{GSM8K} sample~3.}
This example shows a semantic bifurcation. Branches that continue with
\rawfrag{total cost} are more likely to preserve the repair-cost term, while
branches that continue with \rawfrag{new value} often omit it and terminate
near \(120{,}000\).}
\label{tab:gsm8k-token-bifurcation-sample-3}
\begin{tabularx}{\linewidth}{@{}lYYccr@{}}
\toprule
Method & First divergent continuation & Terminal answer span
& Pred. & EM & NFE \\
\midrule
Vanilla
& \rawfrag{new value of the house}
& \rawfrag{Therefore, Josh's profit is \$120,000.}
& \Invalid & \No & 256 \\

BS=4
& \rawfrag{total cost of the house}
& \rawfrag{incurred a loss. \#\#\#\# -10,000}
& \(-10{,}000\) & \No & 133 \\

BS=8
& \rawfrag{total cost of the house}
& \rawfrag{Josh's profit is \$70,000. \#\#\#\# 70000}
& \(70000\) & \Yes & 96 \\

BS=16
& \rawfrag{new value of the house}
& \rawfrag{Josh made a profit of \$120,000.}
& \Invalid & \No & 106 \\

BS=32
& \rawfrag{new value of the house}
& \rawfrag{Josh's profit is \$120,000.}
& \Invalid & \No & 96 \\

BS=64
& \rawfrag{new value of the house}
& \rawfrag{Josh made a profit of \$10,000.}
& \Invalid & \No & 123 \\

BS=128
& \rawfrag{new value of the house}
& \rawfrag{Therefore, Josh's profit is \$120,000.}
& \Invalid & \No & 101 \\

BlockBatch
& \rawfrag{total cost of the house}
& \rawfrag{Josh made a profit of \$100,000.}
& \Invalid & \No & 330 \\
\bottomrule
\end{tabularx}
\end{table}
\FloatBarrier

\paragraph{Takeaway.}
Across the three examples, the first bifurcation occurs early: token
positions~26, 11, and~13 after whitespace normalization. Samples~1 and~2 show
mostly formatting bifurcation: several branches compute the correct terminal
number but fail strict exact match because they emit \rawfrag{\textbackslash{}boxed\{\}}
or omit the \rawfrag{\#\#\#\#} marker. Sample~3 shows reasoning bifurcation:
the first divergent continuation already separates branches that track total
cost from branches that focus on new value, and this early split determines
whether the repair cost is preserved in the terminal answer. Therefore, raw
token bifurcation is useful for distinguishing extraction artifacts from
genuine reasoning drift. BlockBatch benefits from branch diversity only when
the selected terminal branch preserves the correct arithmetic trajectory and
emits an extractor-compatible final answer.
\section{Appendix: KV Cache Vector Space Formulation}
\label{app:kv-cache-vector-space}
\subsection{Basic Operations}
\label{app:basic-operations}

We model each branch cache as a fixed vectorization of all key and value tensors
across layers, heads, and sequence positions. Thus, at event step \(t\), branch
\(b\) has
$$
K_t^{(b)}\in\mathbb{R}^{D}.
$$
All norms below are Euclidean norms on this fixed vectorization. The event index
\(t\) ranges over block-denoise and full-refresh events.

Let
$$
x_t^{(b)}\in\widetilde{\mathcal V}^{L}
$$
be the current token state, where \(\widetilde{\mathcal V}\) contains the
vocabulary and the mask token. Let
$$
F(x_t^{(b)})\in\mathbb{R}^{D}
$$
denote the full-sequence KV-cache map obtained by recomputing the cache from
the complete current sequence.

We distinguish two cache-update operations.

\begin{enumerate}
\item \textbf{Full-sequence refresh.}
A full refresh recomputes the cache from the complete current token state:
$$
K_{t+1}^{(b)}=F(x_t^{(b)}).
$$
The corresponding cache correction is
$$
\Delta_{\mathrm{full},t}^{(b)}
=
F(x_t^{(b)})-K_t^{(b)}.
$$

\item \textbf{Block denoise.}
A block-denoise step updates an active block
\(B_t^{(b)}\subseteq\{1,\ldots,L\}\), with
\(|B_t^{(b)}|=s_b\), while reusing cached context outside the block:
$$
K_{t+1}^{(b)}
=
K_t^{(b)}
+
\Delta_{\mathrm{blk},t}^{(b)}.
$$
\end{enumerate}

This appendix analyzes how the active block size affects the expected
magnitude of the local KV correction. Synchronization and branch selection are
policy operations in BlockBatch and are not needed for the following KV-norm
bounds.

\subsection{Proposition 1: Expected Block-Local KV Movement}
\label{prop:moreUpdateFromFullRefresh}

The following result formalizes the average-case locality of block denoising.
The \(\sqrt{m/L}\) factor does not hold for every fixed block and every fixed
vector. It holds in expectation when a block of size \(m\) covers an \(m/L\)
fraction of positions on average.

\paragraph{Positional KV-energy notation.}
For a flattened KV-cache perturbation \(v\in\mathbb{R}^{D}\), write
$$
v=(v_1,\ldots,v_L),
$$
where \(v_p\) contains all key/value coordinates associated with sequence
position \(p\), across all layers, heads, and key/value channels. For a block
\(B\subseteq\{1,\ldots,L\}\), let \(P_B\) be the orthogonal projection that
keeps coordinates associated with positions in \(B\) and zeros all other
coordinates. Then
$$
\|P_Bv\|_2^2
=
\sum_{p\in B}\|v_p\|_2^2.
$$
Here ``energy'' means squared Euclidean norm of the KV perturbation.

\begin{lemma}[Expected block-projection energy]
Let \(B\) be a random block of size \(m\) such that every position has the same
inclusion probability:
$$
\Pr(p\in B)=\frac{m}{L},
\qquad p=1,\ldots,L.
$$
Then, for any fixed KV perturbation \(v\in\mathbb{R}^{D}\),
$$
\mathbb{E}_B\left[\|P_Bv\|_2^2\right]
=
\frac{m}{L}\|v\|_2^2.
$$
Consequently,
$$
\mathbb{E}_B\left[\|P_Bv\|_2\right]
\leq
\sqrt{\frac{m}{L}}\|v\|_2.
$$
\end{lemma}

\begin{proof}
By definition,
$$
\|P_Bv\|_2^2
=
\sum_{p=1}^{L}\mathbf{1}\{p\in B\}\|v_p\|_2^2.
$$
Taking expectation over \(B\),
$$
\begin{aligned}
\mathbb{E}_B\left[\|P_Bv\|_2^2\right]
&=
\sum_{p=1}^{L}
\mathbb{E}_B[\mathbf{1}\{p\in B\}]
\|v_p\|_2^2 \\
&=
\sum_{p=1}^{L}
\Pr(p\in B)\|v_p\|_2^2 \\
&=
\frac{m}{L}\sum_{p=1}^{L}\|v_p\|_2^2 \\
&=
\frac{m}{L}\|v\|_2^2.
\end{aligned}
$$
By Jensen's inequality,
$$
\mathbb{E}_B[\|P_Bv\|_2]
\leq
\sqrt{
\mathbb{E}_B[\|P_Bv\|_2^2]
}
=
\sqrt{\frac{m}{L}}\|v\|_2.
$$
\end{proof}

\begin{proposition}[Expected block-local KV update bound]
Let \(m=|B_t^{(b)}|\) and let \(L\) be the sequence length. Consider the
full-refresh correction
$$
\Delta_{\mathrm{full},t}^{(b)}
=
F(x_t^{(b)})-K_t^{(b)}.
$$
Decompose the block-denoise correction as
$$
\Delta_{\mathrm{blk},t}^{(b)}
=
P_{B_t^{(b)}}\Delta_{\mathrm{full},t}^{(b)}
+
r_t^{(b)}.
$$
The residual \(r_t^{(b)}\) is defined by this equation and captures the part of
the block-denoise correction that is not explained by the block-restricted
full-refresh correction.

If the active block satisfies
$$
\Pr(p\in B_t^{(b)})=\frac{m}{L},
\qquad p=1,\ldots,L,
$$
then
$$
\mathbb{E}_{B}
\left[
\left\|P_{B_t^{(b)}}\Delta_{\mathrm{full},t}^{(b)}\right\|_2^2
\right]
=
\frac{m}{L}
\left\|\Delta_{\mathrm{full},t}^{(b)}\right\|_2^2.
$$
Therefore,
$$
\mathbb{E}_{B}
\left[
\left\|P_{B_t^{(b)}}\Delta_{\mathrm{full},t}^{(b)}\right\|_2
\right]
\leq
\sqrt{\frac{m}{L}}
\left\|\Delta_{\mathrm{full},t}^{(b)}\right\|_2.
$$
With the residual included,
$$
\mathbb{E}_{B}
\left[
\left\|\Delta_{\mathrm{blk},t}^{(b)}\right\|_2
\right]
\leq
\sqrt{\frac{m}{L}}
\left\|\Delta_{\mathrm{full},t}^{(b)}\right\|_2
+
\mathbb{E}_{B}\left[\left\|r_t^{(b)}\right\|_2\right].
$$
\end{proposition}

\begin{proof}
Apply the expected block-projection lemma to
\(v=\Delta_{\mathrm{full},t}^{(b)}\):
$$
\mathbb{E}_{B}
\left[
\left\|P_{B_t^{(b)}}\Delta_{\mathrm{full},t}^{(b)}\right\|_2^2
\right]
=
\frac{m}{L}
\left\|\Delta_{\mathrm{full},t}^{(b)}\right\|_2^2.
$$
Jensen's inequality gives
$$
\mathbb{E}_{B}
\left[
\left\|P_{B_t^{(b)}}\Delta_{\mathrm{full},t}^{(b)}\right\|_2
\right]
\leq
\sqrt{\frac{m}{L}}
\left\|\Delta_{\mathrm{full},t}^{(b)}\right\|_2.
$$
Using
$$
\Delta_{\mathrm{blk},t}^{(b)}
=
P_{B_t^{(b)}}\Delta_{\mathrm{full},t}^{(b)}
+
r_t^{(b)}
$$
and the triangle inequality,
$$
\left\|\Delta_{\mathrm{blk},t}^{(b)}\right\|_2
\leq
\left\|P_{B_t^{(b)}}\Delta_{\mathrm{full},t}^{(b)}\right\|_2
+
\left\|r_t^{(b)}\right\|_2.
$$
Taking expectation over \(B\) yields
$$
\mathbb{E}_{B}
\left[
\left\|\Delta_{\mathrm{blk},t}^{(b)}\right\|_2
\right]
\leq
\sqrt{\frac{m}{L}}
\left\|\Delta_{\mathrm{full},t}^{(b)}\right\|_2
+
\mathbb{E}_{B}\left[\left\|r_t^{(b)}\right\|_2\right].
$$
\end{proof}

\paragraph{Comparison with full refresh.}
A full refresh applies the complete correction
$$
\Delta_{\mathrm{full},t}^{(b)}
=
F(x_t^{(b)})-K_t^{(b)}.
$$
A block-denoise step applies only the active-block component of this correction,
up to residual effects. Therefore, if
$$
\mathbb{E}_{B}\left[\left\|r_t^{(b)}\right\|_2\right]
<
\left(1-\sqrt{\frac{m}{L}}\right)
\left\|\Delta_{\mathrm{full},t}^{(b)}\right\|_2,
$$
then
$$
\mathbb{E}_{B}
\left[
\left\|\Delta_{\mathrm{blk},t}^{(b)}\right\|_2
\right]
<
\left\|\Delta_{\mathrm{full},t}^{(b)}\right\|_2.
$$

\paragraph{Interpretation.}
This proposition uses one averaging condition: each token position is included
in an active block with probability \(m/L\). Under that condition, the direct
block-local component contains an \(m/L\) fraction of squared KV-update energy
in expectation. The norm therefore scales by \(\sqrt{m/L}\). The residual
\(r_t^{(b)}\) covers effects not captured by a pure projection of the
full-refresh correction, such as attention spillover outside the active block.
Thus the result gives an average-case explanation for why block-denoise events
produce smaller KV movements than full-sequence refreshes, without requiring
extra concentration, smoothness, or interpolation assumptions.

\subsection{Proposition 2: Periodic Refresh Stabilizes Block Denoising}
\label{refresh_stablize}

Let
$$
F:\widetilde{\mathcal V}^{L}\rightarrow \mathbb{R}^{D}
$$
denote the full-sequence KV-cache map on discrete token states. The reachable
full-refresh cache set is
$$
\mathcal{C}
=
\{F(x):x\in\widetilde{\mathcal V}^{L}\}.
$$
We call \(\mathcal{C}\) the reachable full-refresh cache set rather than a
smooth manifold, since \(\widetilde{\mathcal V}^{L}\) is discrete.

For branch \(b\), define the ideal cache associated with its current token
state:
$$
K_{\star,t}^{(b)}
=
F(x_t^{(b)}),
$$
and define the cache-consistency error:
$$
E_t^{(b)}
=
\left\lVert
K_t^{(b)}-K_{\star,t}^{(b)}
\right\rVert_2
=
\left\lVert
K_t^{(b)}-F(x_t^{(b)})
\right\rVert_2.
$$
This measures the deviation between the maintained approximate cache and the
cache obtained by a full forward pass on the same token state.

\begin{proposition}[Refresh-stabilized cache error under bounded drift]
Suppose each block-denoise step satisfies
$$
E_{t+1}^{(b)}
\leq
(1+\lambda_b)E_t^{(b)}+\epsilon_B,
\qquad
\lambda_b\geq0,\quad \epsilon_B\geq0,
$$
and each full refresh satisfies
$$
E_{t+1}^{(b)}
\leq
\beta E_t^{(b)}+\epsilon_F,
\qquad
0\leq\beta<1,\quad \epsilon_F\geq0.
$$
For an exact full refresh with unchanged token state and exact arithmetic, this
corresponds to the special case \(\beta=0\) and \(\epsilon_F=0\).

Suppose one full refresh is applied after every \(R\) block-denoise steps. Let
\(t_n\) denote the event index immediately after the \(n\)-th full refresh and
define \(Y_n=E_{t_n}^{(b)}\). Let
$$
a_b=1+\lambda_b,\qquad
\rho=\beta a_b^R,
$$
and
$$
c=
\beta\epsilon_B\sum_{i=0}^{R-1}a_b^i+\epsilon_F.
$$
If
$$
\rho=\beta(1+\lambda_b)^R<1,
$$
then the refresh-boundary errors satisfy
$$
Y_n
\leq
\rho^nY_0
+
c\frac{1-\rho^n}{1-\rho}.
$$
Consequently,
$$
\limsup_{n\rightarrow\infty}Y_n
\leq
\frac{c}{1-\rho}.
$$
Furthermore, for any within-cycle step \(0\leq r<R\),
$$
E_{t_n+r}^{(b)}
\leq
a_b^rY_n
+
\epsilon_B\sum_{i=0}^{r-1}a_b^i.
$$
\end{proposition}

\begin{proof}
Starting immediately after a refresh at time \(t_n\), after \(R\)
block-denoise steps,
$$
E_{t_n+R}^{(b)}
\leq
a_b^R E_{t_n}^{(b)}
+
\epsilon_B\sum_{i=0}^{R-1}a_b^i.
$$
Applying one refresh gives
$$
\begin{aligned}
E_{t_{n+1}}^{(b)}
&\leq
\beta E_{t_n+R}^{(b)}
+
\epsilon_F \\
&\leq
\beta a_b^R E_{t_n}^{(b)}
+
\beta\epsilon_B\sum_{i=0}^{R-1}a_b^i
+
\epsilon_F.
\end{aligned}
$$
Thus,
$$
Y_{n+1}
\leq
\rho Y_n+c.
$$
Iterating the scalar recurrence yields
$$
Y_n
\leq
\rho^nY_0
+
c\sum_{j=0}^{n-1}\rho^j
=
\rho^nY_0
+
c\frac{1-\rho^n}{1-\rho}.
$$
Since \(\rho<1\), taking \(n\to\infty\) gives
$$
\limsup_{n\rightarrow\infty}Y_n
\leq
\frac{c}{1-\rho}.
$$
The within-cycle bound follows by applying the block-denoise recurrence for
\(r\) steps starting from \(Y_n\).
\end{proof}

\paragraph{Interpretation.}
This result is a stability lemma conditional on bounded block drift and refresh
contraction. It states a sufficient condition under which periodic full refreshes
prevent cache-consistency error from growing without bound.

\subsection{Proposition 3: Projected Cache Separation Diagnostics}
\label{prop:Branch_bifurcation_tanspace}

We now define the projected coordinate system used in the KV-cache diagnostics.
Let \(c_0\in\mathbb{R}^{D}\) be a fixed nonzero reference cache vector, such as
the shared prefill cache or the mean initial branch cache. Define
$$
u_0=\frac{c_0}{\lVert c_0\rVert_2}.
$$
Let \(e_1,e_2\in u_0^\perp\) be orthonormal tangent directions. In practice,
\(e_1,e_2\) may be chosen by PCA, a deterministic sketch, or another fixed
branch-independent projection rule. Let
$$
U=[u_0,e_1,e_2]\in\mathbb{R}^{D\times 3}.
$$

For each branch cache, decompose
\begin{equation}
\small
K_t^{(b)} - c_0
=
z_t^{(b)}u_0
+
a_t^{(b)}e_1
+
q_t^{(b)}e_2
+
h_t^{(b)} .
\end{equation}
where
\begin{equation}
\small
\begin{gathered}
z_t^{(b)}=\langle K_t^{(b)}-c_0,u_0\rangle,\quad
a_t^{(b)}=\langle K_t^{(b)}-c_0,e_1\rangle,\\
q_t^{(b)}=\langle K_t^{(b)}-c_0,e_2\rangle,\quad
h_t^{(b)}=(I-UU^\top)(K_t^{(b)}-c_0).
\end{gathered}
\end{equation}
By construction,
\begin{equation}
\small
h_t^{(b)}
\perp
\operatorname{span}\{u_0,e_1,e_2\}.
\end{equation}

For two branches \(i,j\), define the full cache distance
$$
d_{ij,t}^2
=
\left\lVert
K_t^{(i)}-K_t^{(j)}
\right\rVert_2^2
$$
and the projected cache distance
\begin{equation}
\small
\begin{aligned}
d_{\mathrm{proj},ij,t}^2
&=
\left|z_t^{(i)}-z_t^{(j)}\right|^2
+
\left|a_t^{(i)}-a_t^{(j)}\right|^2 \\
&\quad+
\left|q_t^{(i)}-q_t^{(j)}\right|^2 .
\end{aligned}
\end{equation}

\begin{proposition}[Orthogonal cache-distance decomposition]
Under the orthogonal decomposition above,
$$
\begin{aligned}
\left\lVert K_t^{(i)}-K_t^{(j)}\right\rVert_2^2
&=
\left|z_t^{(i)}-z_t^{(j)}\right|^2
+
\left|a_t^{(i)}-a_t^{(j)}\right|^2 \\
&\quad+
\left|q_t^{(i)}-q_t^{(j)}\right|^2
+
\left\lVert h_t^{(i)}-h_t^{(j)}\right\rVert_2^2.
\end{aligned}
$$
Therefore,
$$
d_{\mathrm{proj},ij,t}
\leq
d_{ij,t}.
$$
\end{proposition}

\begin{proof}
The vectors \(u_0,e_1,e_2\) are orthonormal and each residual
\(h_t^{(b)}\) lies in the orthogonal complement of their span. The result
therefore follows from the Pythagorean theorem.
\end{proof}

We define average branch dispersion as
$$
D_t
=
\frac{1}{B(B-1)}
\sum_{i\neq j}
\left\lVert K_t^{(i)}-K_t^{(j)}\right\rVert_2^2.
$$
The projected coordinates \((z,a,q)\) provide a low-dimensional view of one
component of the full branch dispersion. Therefore, separation in the plotted
\((e_1,e_2,z)\) coordinates is a diagnostic for KV-cache trajectory divergence,
not a proof by itself that token predictions must differ.

\subsection{Case Study: KV Cache Vector Space Properties}
\label{app:kv-cache-case-study}

To examine whether the geometric claims of
Sec.~\ref{app:kv-cache-vector-space} hold across prompts, we present three
\texttt{HumanEval} case studies: samples~2, 15, and 19. For each sample, we
show two paired diagnostics: (i)~the projected KV-cache trajectory in the
$(e_1,e_2,z)$ coordinate system of Sec.~\ref{app:basic-operations}, where the
axial coordinate is $z=\langle K-c_0,u_0\rangle$ with
$u_0=c_0/\lVert c_0\rVert_2$; and (ii)~the token-level consensus heatmap
across the six block-size branches that share the prompt and prefill cache.
The KV diagnostic reflects geometric behavior of the cache vector, while the
consensus heatmap exposes the token-level outcome of that geometry at decoded
positions.

Across the three prompts, we observe the same three phenomena predicted by
Props.~\ref{prop:moreUpdateFromFullRefresh}--\ref{prop:Branch_bifurcation_tanspace}:
adjacent block-denoise events remain confined to small tangent-plane
neighborhoods; every large axial excursion in $z$ is initiated by a
full-sequence refresh rather than by local denoising; and branches that share
a common prefix fan out in $(e_1,e_2)$ at the same generation step where the
consensus heatmap records the first divergent token. This consistency
suggests that the conic-manifold view of KV-cache evolution is not an
artifact of a single prompt and that token-level bifurcation can be read off
the KV geometry directly.

\captionsetup[subfigure]{font=small,labelfont=bf}
\captionsetup[figure]{font=small,labelfont=bf}

\FloatBarrier
\begin{figure*}[!p]
  \centering
  \setlength{\abovecaptionskip}{3pt}
  \setlength{\belowcaptionskip}{-2pt}

  \begin{subfigure}[t]{0.94\textwidth}
    \centering
    \includegraphics[width=0.94\linewidth]{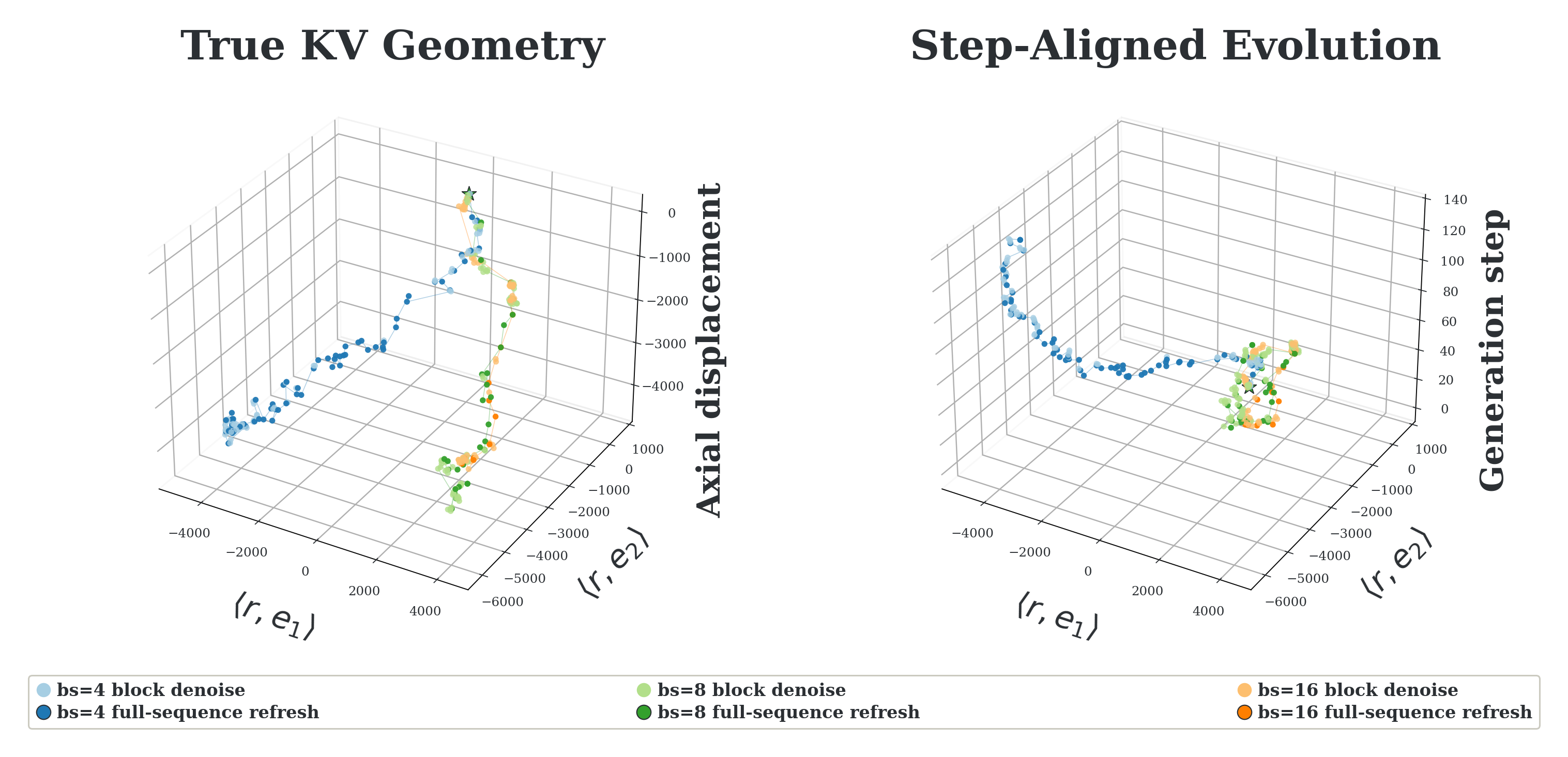}
    \caption{KV-cache trajectory: true KV geometry and step-aligned evolution.}
    \label{fig:kv-humaneval-2-trajectory}
  \end{subfigure}

  \vspace{0.35em}

  \begin{subfigure}[t]{0.86\textwidth}
    \centering
    \includegraphics[width=0.86\linewidth]{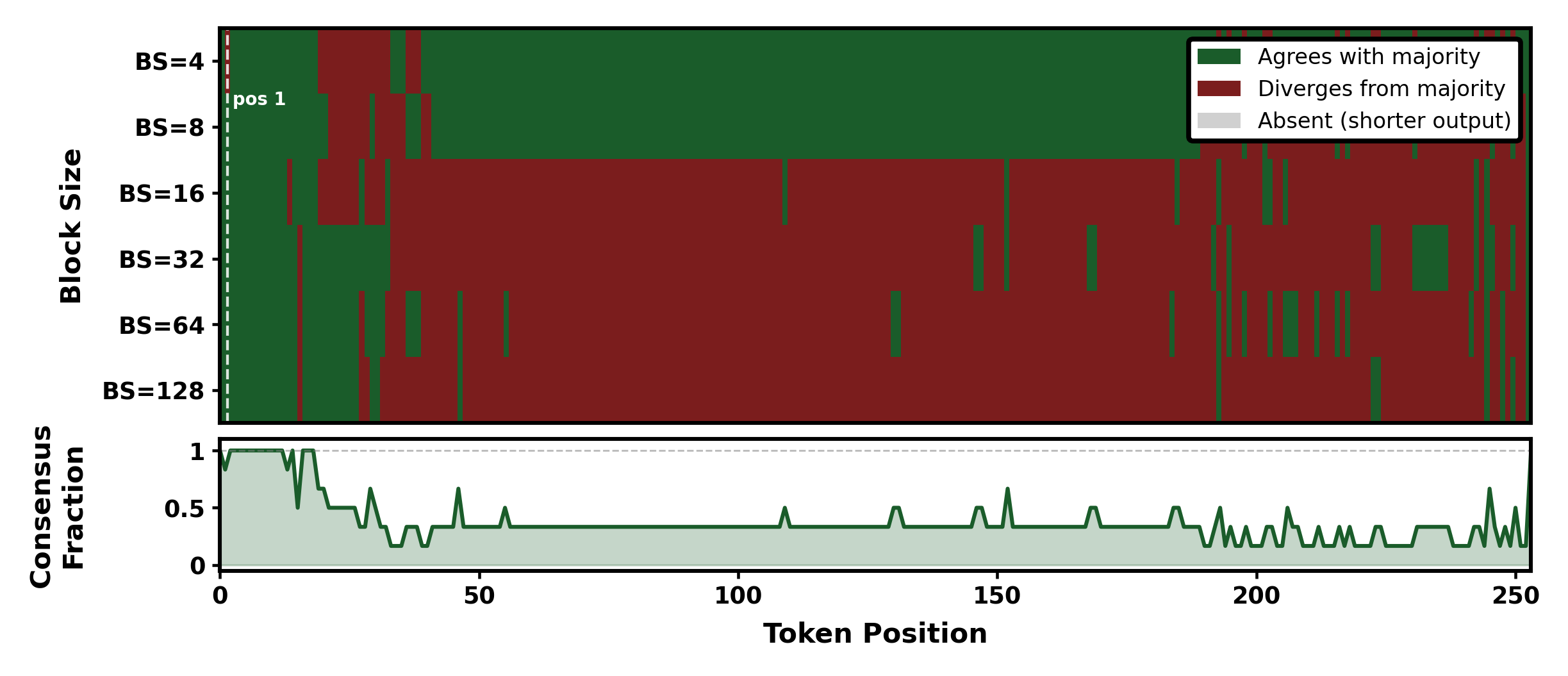}
    \caption{Block-size branch consensus heatmap.}
    \label{fig:consensus-humaneval-2}
  \end{subfigure}

  \vspace{-0.1em}

  \caption{\textbf{Case study: \texttt{HumanEval} sample~2.}
  States from all six block-size branches concentrate along a low-dimensional
  locus dominated by $u_0$. Block-denoise events form dense tangent clusters,
  whereas large axial jumps are initiated only by full-sequence refreshes
  (Prop.~\ref{prop:moreUpdateFromFullRefresh}). The consensus heatmap shows
  that all branches share an identical prefix and then bifurcate at the dashed
  line, matching the tangent-plane separation visible in the KV trajectory
  (Prop.~\ref{prop:Branch_bifurcation_tanspace}).}
  \label{fig:case-humaneval-2}
\end{figure*}
\FloatBarrier

\FloatBarrier
\begin{figure*}[!p]
  \centering
  \setlength{\abovecaptionskip}{3pt}
  \setlength{\belowcaptionskip}{-2pt}

  \begin{subfigure}[t]{0.94\textwidth}
    \centering
    \includegraphics[width=0.94\linewidth]{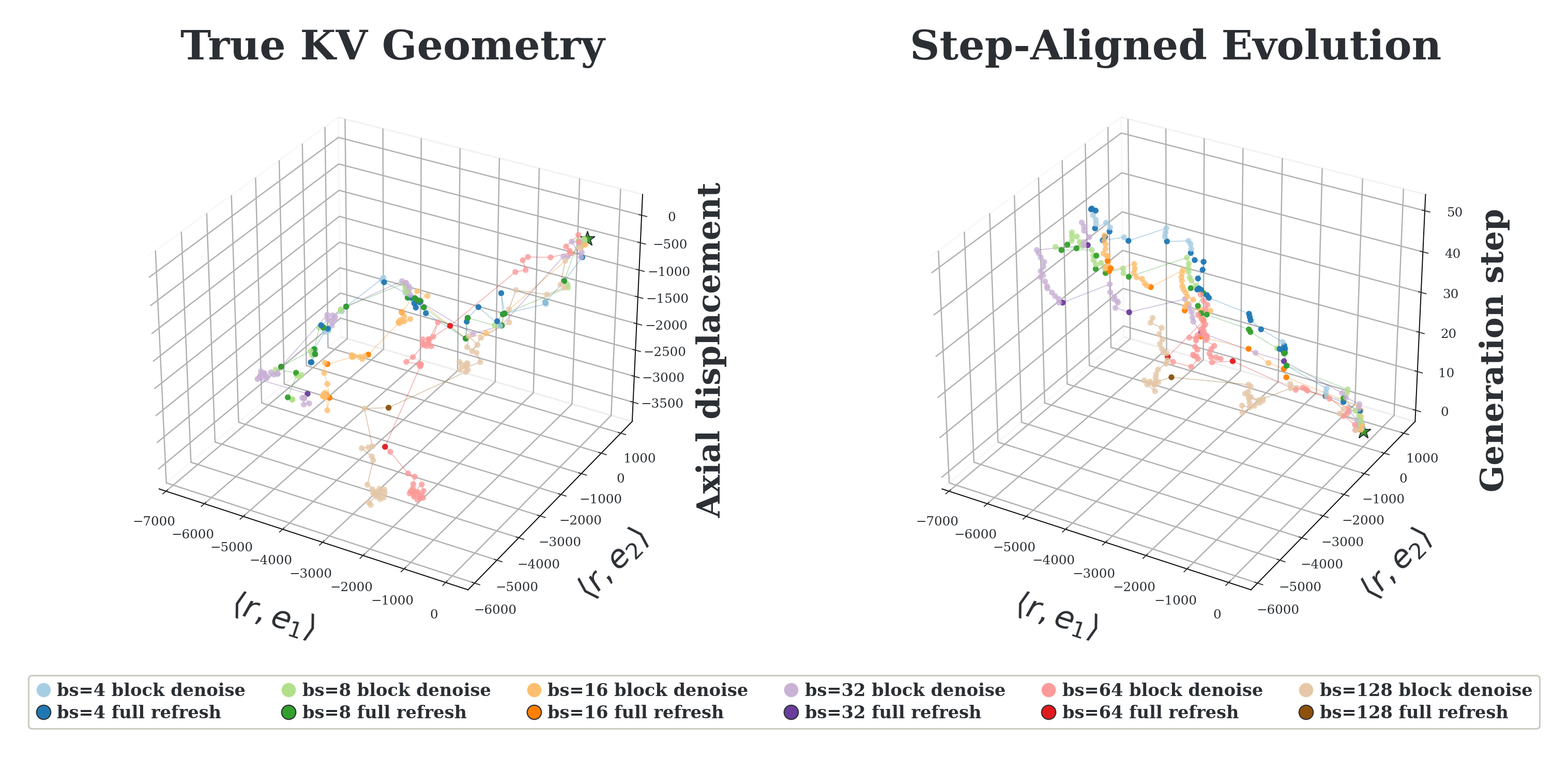}
    \caption{KV-cache trajectory: true KV geometry and step-aligned evolution.}
    \label{fig:kv-humaneval-15-trajectory}
  \end{subfigure}

  \vspace{0.35em}

  \begin{subfigure}[t]{0.86\textwidth}
    \centering
    \includegraphics[width=0.86\linewidth]{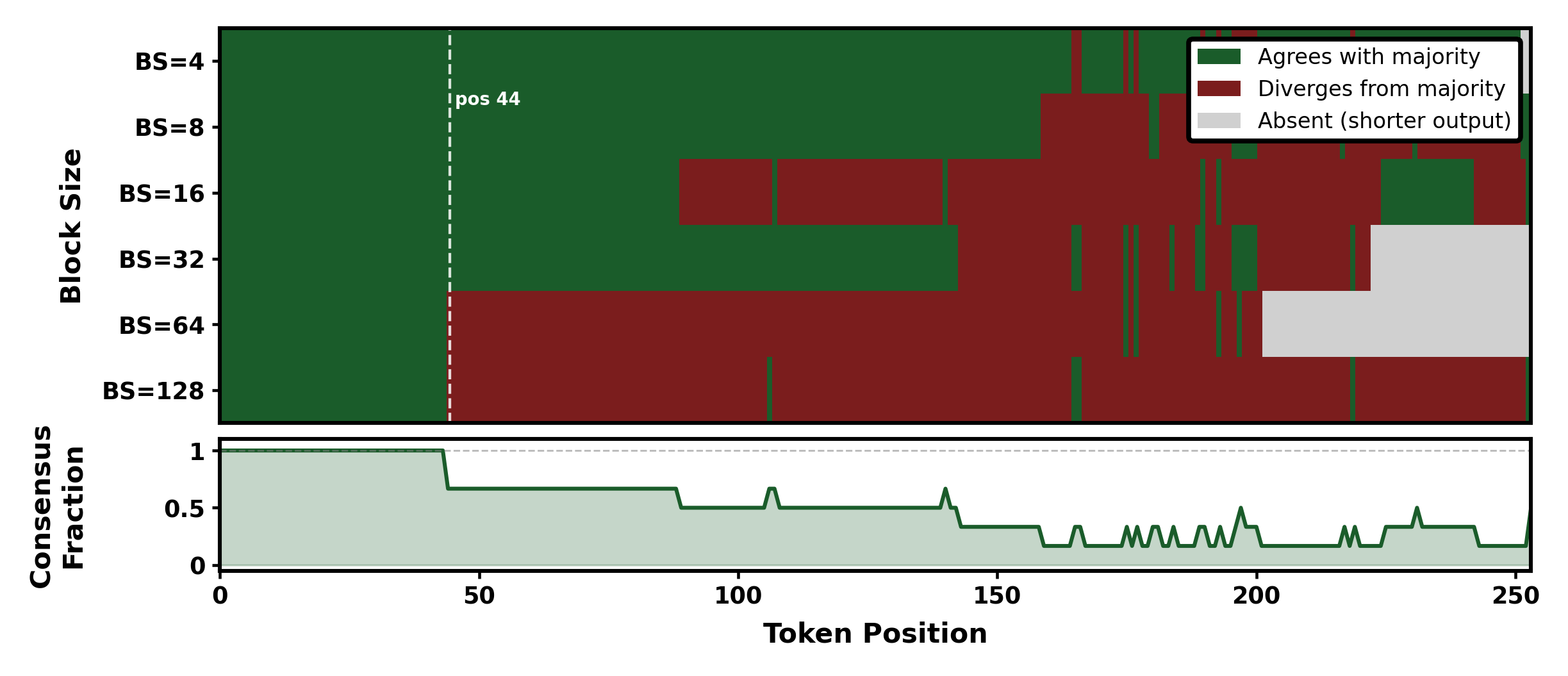}
    \caption{Block-size branch consensus heatmap.}
    \label{fig:consensus-humaneval-15}
  \end{subfigure}

  \vspace{-0.1em}

  \caption{\textbf{Case study: \texttt{HumanEval} sample~15.}
  The same local-region-with-refresh-jumps pattern appears on a longer prompt.
  Tangent spread accumulates with generation step, while the consensus panel
  shows progressive divergence of larger block sizes first, consistent with the
  Lipschitz block bound of Prop.~\ref{prop:moreUpdateFromFullRefresh}.}
  \label{fig:case-humaneval-15}
\end{figure*}
\FloatBarrier

\FloatBarrier
\begin{figure*}[!p]
  \centering
  \setlength{\abovecaptionskip}{3pt}
  \setlength{\belowcaptionskip}{-2pt}

  \begin{subfigure}[t]{0.94\textwidth}
    \centering
    \includegraphics[width=0.94\linewidth]{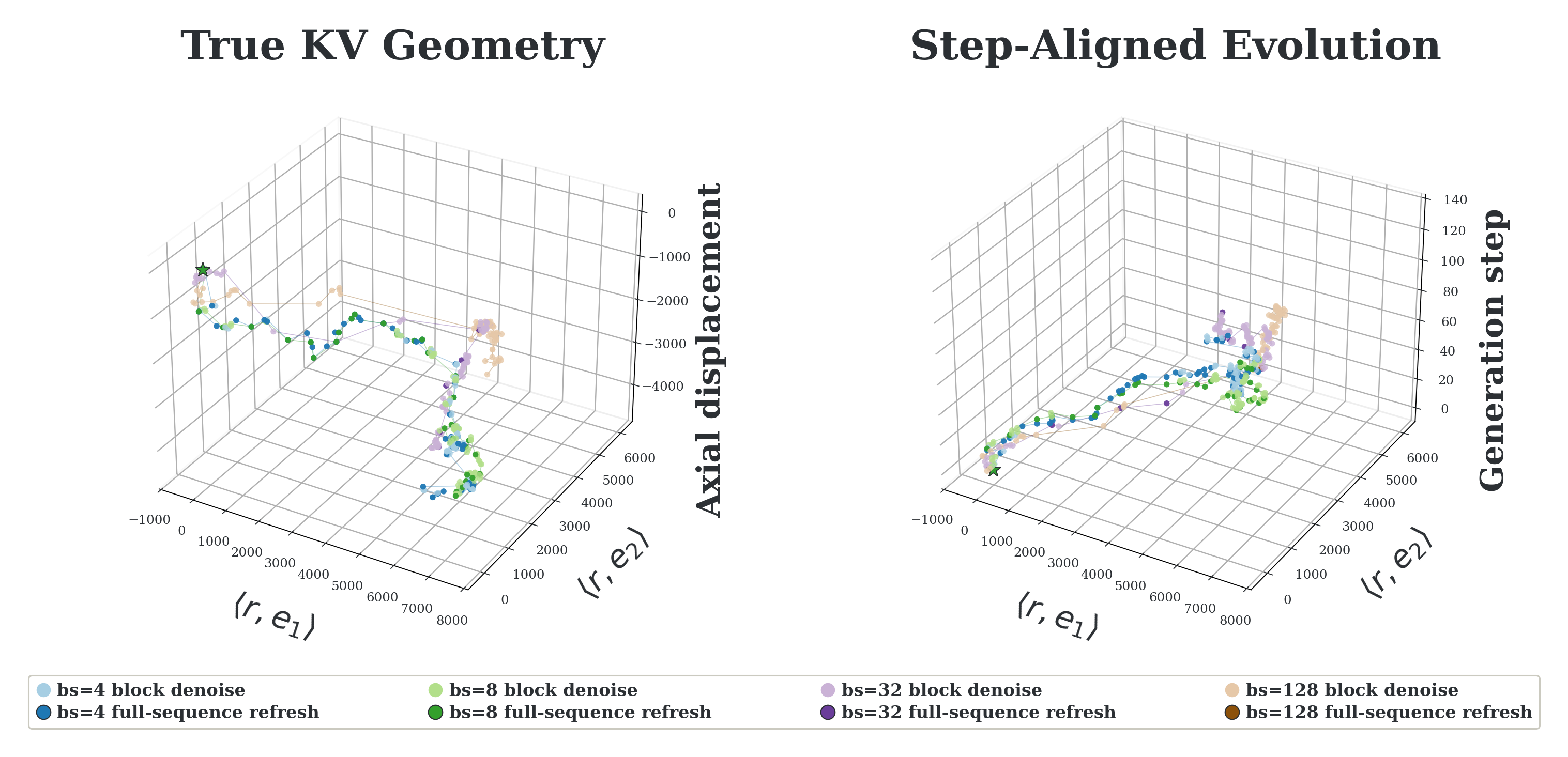}
    \caption{KV-cache trajectory: true KV geometry and step-aligned evolution.}
    \label{fig:kv-humaneval-19-trajectory}
  \end{subfigure}

  \vspace{0.35em}

  \begin{subfigure}[t]{0.86\textwidth}
    \centering
    \includegraphics[width=0.86\linewidth]{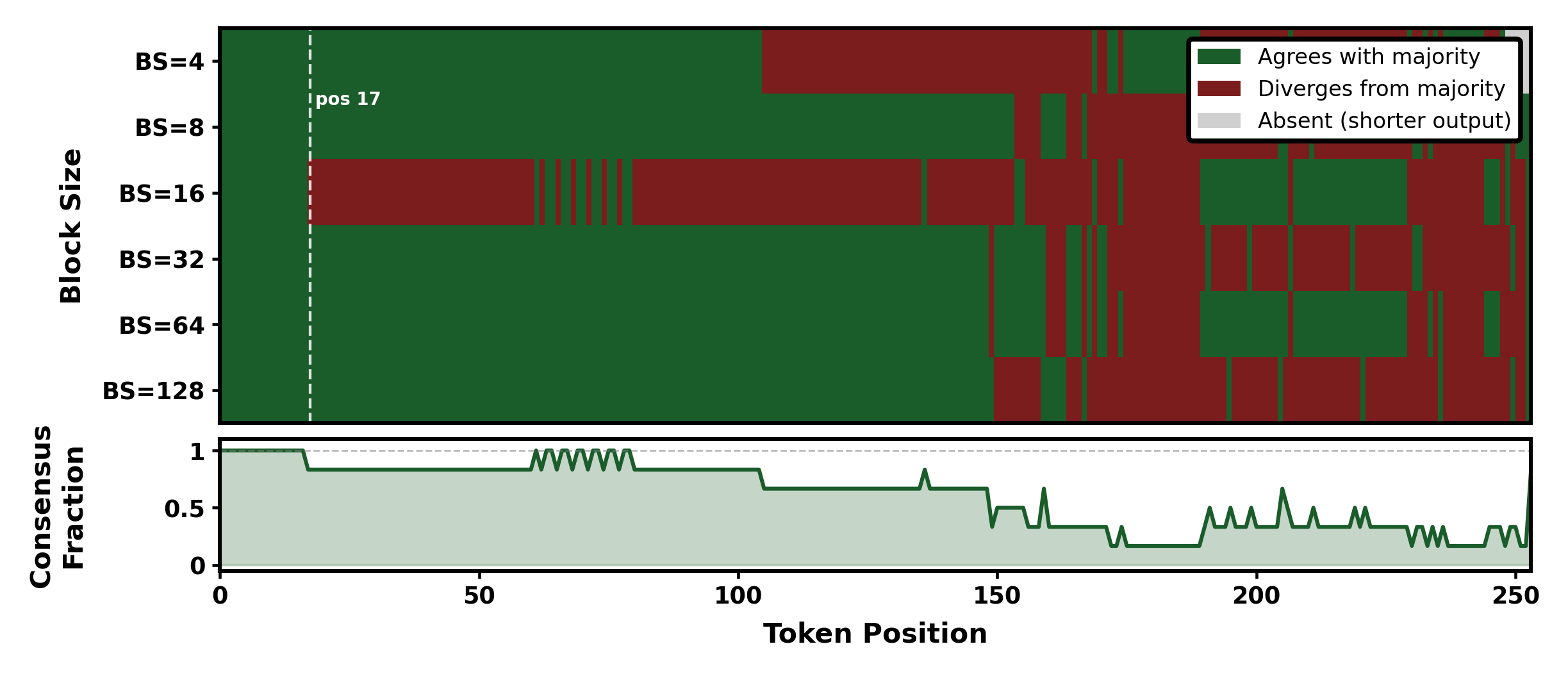}
    \caption{Block-size branch consensus heatmap.}
    \label{fig:consensus-humaneval-19}
  \end{subfigure}

  \vspace{-0.1em}

  \caption{\textbf{Case study: \texttt{HumanEval} sample~19.}
  This harder prompt exhibits earlier bifurcation. The tangent-plane fan-out in
  the KV trajectory and the early color split in the consensus heatmap occur at
  the same generation step, again confirming that branch bifurcation is the
  token-level signature of tangent-space separation
  (Prop.~\ref{prop:Branch_bifurcation_tanspace}).}
  \label{fig:case-humaneval-19}
\end{figure*}
\FloatBarrier

\paragraph{Summary.}
The three case studies provide consistent empirical support for the
step / axial decomposition used throughout
Sec.~\ref{app:kv-cache-vector-space}. The axial direction $u_0$ is moved
almost exclusively by full-sequence refreshes
(Prop.~\ref{prop:moreUpdateFromFullRefresh}), the bounded recurrence of
Prop.~\ref{refresh_stablize} is visible as tight axial banding between refresh
cycles, and the tangent-plane separation predicted by
Prop.~\ref{prop:Branch_bifurcation_tanspace} aligns step-for-step with the
position of first token-level divergence reported by the consensus heatmaps.
Together, these observations justify treating block-size branches as competing
local traversals of a shared KV-cache geometry rather than independent
generations.
\section{Appendix: Block Batching Algorithm}\label{app:block-batching-algorithm}
\subsection{Fused Block Batching with a Unified KV Cache}\label{app:fused-block-batching}

Let $\texttt{[M]}$ denote the mask token, $L_p$ the prompt length, and $G$ the generation length. For each block size $b_k \in \mathcal{B}$, BlockBatch maintains one branch state $S_k$ and one row in the unified token tensor and KV cache.Each branch owns its own token row and KV-cache row; rows are copied across branches only through the merge-and-sync policy.

\begin{algorithm}[!t]
\scriptsize
\caption{\textsc{BlockBatch}: Fused Block Batching with Row-Owned KV Cache}
\label{alg:blockbatch}
\begin{algorithmic}[1]
\REQUIRE Model $f_\theta$, prompt $\mathbf{p}$, generation length $G$,
block sizes $\mathcal{B}=\{b_1,\ldots,b_N\}$, mask token $\texttt{[M]}$,
decode threshold $\tau_{\mathrm{dec}}$, merge threshold $\tau_{\mathrm{conf}}$,
sync threshold $\tau_{\mathrm{sync}}$, refresh interval $R$
\ENSURE Generated sequence $\hat{\mathbf{x}}$

\STATE $L_p \gets |\mathbf{p}|$, \quad $L \gets L_p + G$
\STATE Initialize $\mathbf{X}\in\mathbb{N}^{N\times L}$ with $\texttt{[M]}$
\STATE Set $\mathbf{X}_{k,1:L_p}\gets \mathbf{p}$ for all $k$
\STATE Initialize branch states 
$S_k \gets (k,b_k,\mathrm{start}=L_p,\mathrm{end}=L_p+b_k,
\mathrm{done}=\mathrm{false})$

\STATE $(\mathbf{K},\mathbf{V}),\{\boldsymbol{\ell}^{(k)}\}_{k=1}^N
\gets \textsc{InitialFullForward}(f_\theta,\mathbf{X},\{S_k\})$
\COMMENT{one full forward; KV broadcast to all rows}

\FOR{each active branch $S_k$}
    \STATE $\textsc{DecodeCurrentBlock}
    (\mathbf{X}_{k},\boldsymbol{\ell}^{(k)},S_k,\tau_{\mathrm{dec}})$
\ENDFOR

\STATE $\textsc{AdvanceCompletedBlocks}(\mathbf{X},\{S_k\})$
\STATE $\textsc{MergeSync}
(\mathbf{X},\mathbf{K},\mathbf{V},\{S_k\},\tau_{\mathrm{conf}},\tau_{\mathrm{sync}})$
\STATE $(\mathbf{K},\mathbf{V}) \gets
\textsc{RefreshIfNeeded}
(f_\theta,\mathbf{X},\mathbf{K},\mathbf{V},\{S_k\},R)$

\WHILE{there exists $S_k$ such that $S_k.\mathrm{done}=\mathrm{false}$}
    \STATE $\mathcal{A} \gets
    \textsc{GetActiveBranches}(\mathbf{X},\{S_k\})$
    \COMMENT{unfinished branches with masked tokens in the current block}

    \STATE $\mathbf{Q} \gets
    \textsc{PackActiveBlocks}(\mathbf{X},\mathcal{A})$
    \COMMENT{varlen packing removes padded query tokens}

    \STATE $\{\boldsymbol{\ell}^{(k)}\}_{S_k\in\mathcal{A}}
    \gets
    \textsc{BatchedBlockForward}
    (f_\theta,\mathbf{Q},\mathbf{K},\mathbf{V},\mathcal{A})$

    \FOR{each $S_k\in\mathcal{A}$}
        \STATE $\textsc{DecodeCurrentBlock}
        (\mathbf{X}_{k},\boldsymbol{\ell}^{(k)},S_k,\tau_{\mathrm{dec}})$
        \STATE $\textsc{UpdateProbabilityMap}
        (S_k,\boldsymbol{\ell}^{(k)})$
        \STATE $\textsc{CheckEOS}(S_k,\mathbf{X}_{k},\boldsymbol{\ell}^{(k)})$
    \ENDFOR

    \IF{there exists an EOS-ready branch}
        \RETURN $\textsc{SelectEOSWinner}(\{S_k\},\mathbf{X})$
    \ENDIF

    \STATE $\textsc{AdvanceCompletedBlocks}(\mathbf{X},\{S_k\})$
    \STATE $\textsc{MergeSync}
    (\mathbf{X},\mathbf{K},\mathbf{V},\{S_k\},
    \tau_{\mathrm{conf}},\tau_{\mathrm{sync}})$
    \STATE $(\mathbf{K},\mathbf{V}) \gets
    \textsc{RefreshIfNeeded}
    (f_\theta,\mathbf{X},\mathbf{K},\mathbf{V},\{S_k\},R)$
\ENDWHILE

\STATE $k^\star \gets \arg\max_k S_k.\mathrm{progress}$
\RETURN $\mathbf{X}_{k^\star}$
\end{algorithmic}
\end{algorithm}

\subsection{Merge and Synchronization Policy}\label{app:merge-sync-policy}

\begin{algorithm}[!t]
\scriptsize
\caption{\textsc{MergeSync}: Cross-branch merge and synchronization}
\label{alg:merge-sync}

\begin{minipage}{\linewidth}
\raggedright
\textbf{Input:} Branch states $\mathcal{S}=\{S_1,\ldots,S_N\}$, token matrix $\mathbf{X}$,
unified cache $(\mathbf{K},\mathbf{V})$\\
\textbf{Parameters:} merge threshold $\tau_{\mathrm{conf}}$, sync threshold $\tau_{\mathrm{sync}}$
\end{minipage}
\vspace{0.3em}

\begin{algorithmic}[1]

\STATE $S_{\mathrm{lead}}\gets
\operatorname*{arg\,max}_{S_k\in\mathcal{S}}
S_k.\mathrm{tokens\_decoded}$

\IF{$S_{\mathrm{lead}}.\mathrm{tokens\_decoded}=0$}
    \STATE \textbf{return}
\ENDIF

\STATE $\mathcal{O}\gets \textsc{SortActiveBranches}(\mathcal{S})$
\STATE \textit{Merge step: fill masked positions using compatible branches.}

\FOR{each destination branch $S_d\in\mathcal{O}$}
    \STATE $\mathcal{C}_d\gets
    \textsc{CompatibleSources}(S_d,\mathcal{O})$

    \FOR{each position $i\in \textsc{WindowUnion}(\mathcal{C}_d)$}
        \IF{$\mathbf{X}_{d,i}\neq \texttt{[M]}$}
            \STATE \textbf{continue}
        \ENDIF

        \STATE $\mathcal{C}_{d,i}\gets
        \{S_s\in\mathcal{C}_d:
        \mathbf{X}_{s,i}\neq\texttt{[M]}\}$

        \IF{$\mathcal{C}_{d,i}\neq\emptyset$}
            \STATE $S_{s^\star}\gets
            \operatorname*{arg\,max}_{S_s\in\mathcal{C}_{d,i}}
            P_d(i,\mathbf{X}_{s,i})$

            \STATE $v^\star\gets \mathbf{X}_{s^\star,i}$,
            \quad $p^\star\gets P_d(i,v^\star)$

            \IF{$p^\star>\tau_{\mathrm{conf}}$}
                \STATE $\mathbf{X}_{d,i}\gets v^\star$
                \STATE $S_d.\mathrm{tokens\_merged}
                \gets S_d.\mathrm{tokens\_merged}+1$
            \ENDIF
        \ENDIF
    \ENDFOR

    \IF{$\textsc{CurrentBlockDecoded}(S_d,\mathbf{X})$}
        \STATE $\textsc{AdvanceBlock}(S_d)$
    \ENDIF
\ENDFOR

\STATE \textit{Hard synchronization: copy the leader into lagging branches.}

\FOR{each active branch $S_d\in\mathcal{S}$ with $S_d\neq S_{\mathrm{lead}}$}
    \IF{$S_{\mathrm{lead}}.\mathrm{tokens\_decoded}
    - S_d.\mathrm{tokens\_decoded}
    > \tau_{\mathrm{sync}}$}

        \STATE $\mathbf{X}_{d,:}\gets \mathbf{X}_{\mathrm{lead},:}$
        \STATE $(\mathbf{K}_{d},\mathbf{V}_{d})
        \gets
        (\mathbf{K}_{\mathrm{lead}},\mathbf{V}_{\mathrm{lead}})$
        \STATE $\textsc{RealignBlockWindow}(S_d,\mathbf{X}_{d,:})$
        \STATE $S_d.\mathrm{tokens\_decoded}
        \gets S_{\mathrm{lead}}.\mathrm{tokens\_decoded}$
    \ENDIF
\ENDFOR

\end{algorithmic}
\end{algorithm}

\section{Appendix: Reproducibility, Artifacts, and Compute Details}\label{app:reproducibility}

\subsection{Scientific Artifacts}\label{app:scientific-artifacts}

We use publicly available diffusion language models, benchmark datasets, and evaluation protocols.The model artifacts include LLaDA-Instruct-8B~\cite{nie2025large} and Dream-Base-7B~\cite{ye2025dream}. The benchmark artifacts include GSM8K~\cite{cobbe2021training}, MATH~\cite{hendrycksmath2021}, HumanEval~\cite{chen2021evaluating}, and MBPP~\cite{austin2021program}. We cite the original creators of all models and datasets in the main paper and use them only for research evaluation.

\subsection{Artifact Licenses and Intended Use}\label{app:artifact-licenses}

We use existing models and datasets according to their stated research and evaluation purposes.No new human-subject dataset is collected in this work. The experiments are limited to inference-time evaluation and do not involve training on private, personal, or newly collected user data.

\subsection{Evaluation Setup}
  \label{app:evaluation-setup}

  We evaluate GSM8K, MATH, HumanEval, and MBPP using the standard task prompts and
  metrics. GSM8K uses 5-shot prompting, MATH uses 4-shot prompting, MBPP uses
  3-shot prompting, and HumanEval uses 0-shot prompting. HumanEval and MBPP are
  evaluated with execution-based correctness after code sanitization.

  Unless otherwise stated, all methods use greedy decoding with
  $
  G=256,\qquad \text{batch size}=1,\qquad \tau_{\mathrm{conf}}=0.9.
  $
  We report accuracy, average NFE, and average latency. Latency is measured as
  per-example generation wall-clock time and excludes model loading and metric
  post-processing. For Vanilla decoding, NFE is 256. For Fast-dLLM and
  LocalLeap, NFE is the average number of model forwards used by the decoding
  procedure. For BlockBatch, one batched forward pass counts as one NFE; the
  reported NFE includes initial block denoising, block denoising, and full
  refresh forwards.

  Table \ref{tab:block_batching_compact_b32} compares Vanilla decoding, Fast-dLLM dual cache, LocalLeap, and
  BlockBatch on LLaDA-1.5-8B, LLaDA-Instruct-8B, and Dream-Base-7B.
  Fast-dLLM dual cache is evaluated with block size
  $
  32.
  $
  BlockBatch uses $B=\{4,8,16,32,64,128\}$,
$\tau_{\mathrm{sync}}=8$, $R=32$, and
$\tau_{\mathrm{merge}}=0.5$.
  LocalLeap uses its default script configuration:
  $
  \texttt{threshold}=0.9,\qquad
  \texttt{radius}=4,
  $
  with
  $
  \texttt{relaxed\_threshold}=0.75
  $
  for LLaDA models and
  $
  \texttt{relaxed\_threshold}=0.8
  $
  for Dream-Base-7B.

  For Table~\ref{tab:sync_threshold}, we ablate the BlockBatch synchronization
  threshold over
  $
  \tau_{\mathrm{sync}}\in\{4,8,16,32,64\}.
  $
  All other decoding settings match the Table \ref{tab:block_batching_compact_b32} BlockBatch configuration. The
  table reports GSM8K and HumanEval results for LLaDA-Instruct-8B and
  Dream-Base-7B.

  For Table~\ref{tab:refresh_interval}, we ablate the full-refresh interval over
  $
  R\in\{4,8,16,32,64,128\},
  $
  while fixing
  $
  \tau_{\mathrm{sync}}=8,\qquad
  \mathcal{B}=\{4,8,16,32,64,128\}.
  $
\subsection{Computational Infrastructure and Budget}\label{app:compute-budget}

Latency and throughput profiling experiments were run on NVIDIA H200 GPUs.



\end{document}